\documentclass{article}

\usepackage{microtype}
\usepackage{graphicx}
\usepackage{subfigure}
\usepackage{booktabs} % for professional tables

\usepackage{hyperref}

\usepackage[accepted]{icml2024}

\usepackage{amsmath}
\usepackage{amssymb}
\usepackage{mathtools}
\usepackage{amsthm}

\usepackage[capitalize,noabbrev]{cleveref}

\usepackage{hyperref}
\theoremstyle{plain}
\newtheorem{theorem}{Theorem}[section]
\newtheorem{proposition}[theorem]{Proposition}
\newtheorem{lemma}[theorem]{Lemma}
\newtheorem{corollary}[theorem]{Corollary}
\theoremstyle{definition}

\newtheorem{assumption}[theorem]{Assumption}
\theoremstyle{remark}
\newtheorem{remark}[theorem]{Remark}

\usepackage[textsize=tiny]{todonotes}

\DeclareMathOperator{\tr}{Tr}

\DeclareMathOperator{\Cov}{Cov}
\DeclareMathOperator{\Var}{Var}

\DeclareMathOperator{\dist}{dist}

\newcommand{\sodist}{\Theta}

\newcommand{\id}{\mathrm{Id}}
\newcommand{\SO}{\mathrm{SO}}

\renewcommand{\d}{\mathrm{d}}

\newcommand{\R}{{\mathbb{R}}}

\newcommand{\E}{\mathbb{E}}

\renewcommand{\P}{\mathbb{P}}
\newcommand{\Q}{\mathbb{Q}}

\newcommand{\mc}{\mathcal}

\newcommand{\p}{\varphi}

\newcommand{\mcc}{\mathrm{MCC}}

\newcommand{\iso}{\mathrm{iso}}

\newcommand{\vertiii}[1]{{\left\vert\kern-0.25ex\left\vert\kern-0.25ex\left\vert
#1 
    \right\vert\kern-0.25ex\right\vert\kern-0.25ex\right\vert}}

\newcommand{\eps}{\varepsilon}
\newcommand{\diag}{\mathrm{Diag}}

\DeclareMathOperator*{\argmin}{argmin}

\newcommand{\indistribution}{\stackrel{\mc{D}}{=}}
\usepackage{enumitem}

\icmltitlerunning{Robustness of Nonlinear Representation Learning}

\begin{document}

\twocolumn[
\icmltitle{Robustness of Nonlinear Representation Learning}




\begin{icmlauthorlist}
\icmlauthor{Simon Buchholz}{yyy,zz}
\icmlauthor{Bernhard Schölkopf}{yyy,zz,zzz}
\end{icmlauthorlist}

\icmlaffiliation{yyy}{Max Planck Institute for Intelligent Systems,
Tübingen, Germany}
\icmlaffiliation{zz}{Tübingen AI Center, Tübingen, Germany}
\icmlaffiliation{zzz}{ELLIS Institute, Tübingen, Germany}

\icmlcorrespondingauthor{Simon Buchholz}{sbuchholz@tue.mpg.de}

\icmlkeywords{Machine Learning,representation learning, ICA, misspecification, identifiability}

\vskip 0.3in
]



\printAffiliationsAndNotice{}  % leave blank if no need to mention equal contribution

\begin{abstract}
%We consider Independent Component Analysis (ICA) with observations generated according to 
%$x=f(s)=As+h(s)$ where $A$ is an invertible mixing matrix and $h$ a small perturbation.
% While the nonlinear mixing $f$ is not identifiable, we show that
% we can approximately recover the linear part $A$ where the error depends on the size of the perturbation. The main application and motivation for this result concerns the case of undercomplete ICA where the mixing is close to a local isometry in a suitable sense.
% By showing approximate identifiability in this case, we pave the way towards
%  rigorous results for unsupervised representation learning for real world data.
  We study the problem of unsupervised representation learning in slightly misspecified settings, and thus formalize the study of robustness of nonlinear representation learning. We focus on the case where the mixing is close to a local isometry in a suitable distance and show based on existing rigidity results that the mixing can be identified up to linear transformations and small errors. In a second step, we investigate Independent Component Analysis (ICA) with observations generated according to $x=f(s)=As+h(s)$ where $A$ is an invertible mixing matrix and $h$ a small perturbation. We show that we can approximately recover the matrix $A$ and the independent components. Together, these two results show approximate identifiability of nonlinear ICA with almost isometric mixing functions. Those results are a step towards identifiability results for unsupervised representation learning for real-world data that do not follow restrictive model classes.
\end{abstract}
%\section{TODO LIST}
%\begin{enumerate}
%\item Add references
%\item Add GP construction of isometric functions
%\item Check all proofs
%\item Check scaling dependence on $h$ norms
%\item Add text on isometries
%\item Add conclusion
%\item potentially add experiments
%\item add details on ICA algorithm	
%\end{enumerate}
%\vspace{-.3cm}
\section{Introduction}\label{sec:intro}
One of the fundamental problems of data analysis is the unsupervised learning of representations.
Modern machine learning algorithms excel at learning accurate representations of very complex data distributions.
However, those representations are, in general, not related to the true underlying latent factors of variations that generated the data. 
Nevertheless, it is desirable to not only match the training distribution, but also to identify the underlying causal structure because such representations are expected to improve the downstream performance and improve explainability and robustness \citep{scholkopf2021towards}.  

Generally, we can only hope to learn  the ground truth model in identifiable settings, i.e., when, up to certain symmetries, 
there is a unique model in the considered class generating the observations. 
The simplest example of such a setting is linear ICA where we observe linear mixtures of independent variables. Identifiability of linear ICA was  shown   under mild assumptions in \citet{comon}.
On the other hand, it is well known that for general nonlinear functions ICA is not  identifiable \citep{hyvarinen1999nonlinear}. However, recently, a flurry of results was proved that showed various  identifiability results under additional assumptions. Those results rely among others on restrictions of the function class \citep{post_nonlinear,horan2021when,gresele2021independent,buchholz2022function,kivva2022}, multi-environment data \citep{khemakhem2020variational,hyvarinen2019nonlinear}, or interventional data \citep{ahuja2022towards,seigal2022linear,vonkuegelgen2023nonparametric,buchholz2023learning},
for a broader overview we refer to the  recent survey \citet{hyvarinen2023identifiability}. 

While those results brought significant progress to the field of Causal Representation Learning (CRL) they also generally require strong assumptions which will typically only hold approximately for real data. Thus, it is of 
great importance to understand the various robustness properties of 
such identifiability results. A first step in this direction is the investigation of whether the latent variables can still be
approximately  identified  in settings where a small amount of misspecification is allowed, i.e., when the model assumptions are mildly violated. 
Our analysis here focuses on theoretical questions regarding identifiability 
in misspecified settings, but this has nevertheless profound implications for the empirical side because it clarifies what assumptions generate or do not generate useful learning signals that can be exploited by suitable algorithms. This is particularly important since there is still a lack of algorithms 
that uncover the true latent structure for complex data beyond toy settings.

In this work, we study the problem of approximate identifiability in a setting where the mixing function 
is close to a local isometry and the latent variables are independent (ICA).
In particular, our results can be seen as a generalization of the identifiability result of \citet{horan2021when}. The main contributions of this work can be summarized as follows:
%\vspace{-.3cm}
\begin{enumerate}%[noitemsep]
\item We initiate the study of robustness for representation learning and clarify that some existing results are not robust to arbitrarily small amount of misspecification.
\item We show based on existing rigidity results from
material science  that we can identify the latent variables 
up to a linear transformation and a small perturbation when the mixing function $f$ is close to a local isometry.
\item Then we carefully investigate slightly non-linear ICA,  a setting of independent interest. We show that then the linear part of the mixing can be identified up to a small error, which depends on the strength of the perturbation. 
\item 
Finally, we combine the previous two results to show robustness for representation learning problems with independent latent variables and a mixing function close to a local isometry.
\end{enumerate}
The rest of this paper is structured as follows. 
In Section~\ref{sec:setting} we introduce the general setting of representation learning, ICA, and local isometries, and we then consider the robustness of
identifiability in representation learning in Section~\ref{sec:robustness}.
Our results on approximate linear identifiability for approximate local isometries can be found in Section~\ref{sec:approximate_isometries} followed by our analysis of perturbed linear ICA and a combination of the two results in Section~\ref{sec:perturbed}.
We conclude in Section~\ref{sec:conclusion}. In Appendix~\ref{app:notation}
we collect some notation. In this work, $C$ denotes a generic constant that is allowed to change from line to line. We denote by $\indistribution$
equality of the distributions of two random variables.

\section{Setting}\label{sec:setting}

In this section, we introduce and motivate our main setting.
We assume that we have $d$-dimensional latent variables $S$ 
such that their distribution satisfies $\P\in \mc{P}$ for some class of probability distributions $\mc{P}$. For our analysis of local isometries we make the following mild assumption.
\begin{assumption}\label{as:P1}
We assume that the class of probability distributions
$\mc{P}$ has bounded support $\Omega\subset \R^d$
with density lower and upper bounded.
\end{assumption}
In the second part of this work  we   focus on independent component analysis, i.e., we make the following assumption.
\begin{assumption}\label{as:P2}
All $\P\in \mc{P}$ have independent components, i.e.,
%\begin{align}
$ 
 \P=\bigotimes_{i=1}^d \P_i
$
%\end{align}
for some measures $\P_i$ on $\R$ and we furthermore assume that the $\P_i$ are non-Gaussian and have connected support.
\end{assumption}
%TODO add this again!
%We also consider the following generalization of this assumption.
%\begin{assumption}\label{as:P3}
%All $\P\in \mc{P}$ are supported on a rectangle $\Omega=[a_1,b_1]\times\cdots \times [a_d,b_d]$.
%\end{assumption}
The latent variables $S$ are hidden, and we assume that we
observe $X=f(S)$ for some mixing function $f\in \mc{F}(\Omega)$.
Note that the distribution of $X$ is then given by the push-forward $f_\ast\P$.
Here $\mc{F}$ is some function class consisting of functions $f:\Omega\subset\R^d\to \R^D$ and we will always assume without further notice that
$f$ is a diffeomorphism on its image (i.e., injective, differentiable 
and with differentiable inverse on its domain which is a submanifold of $\R^D$).
 Our main interest concerns the function class of local isometries, i.e.,
\begin{align}
\begin{split}
\mc{F}_{\iso}=\{f:\Omega\subset \R^d\to & \R^D:\, \text{ $Df^\top(s) Df(s) = \id$}
\\ &\text{ for all $s\in \Omega$}\}
\end{split}
\end{align}
for some  connected domain $\Omega\subset \R^d$.
The class of local isometries has attracted substantial attention in representation learning because it locally preserves the structure of 
the data, which is a desirable feature in many settings \citep{Tenenbaum2000,donoho_grimes,belkin2003laplacian}.
Closely connected notions like the restricted isometry property play a crucial
role in signal processing \citep{candes2005decoding},
and several works show that making  (parts) of neural networks isometric
improves performance \citep{qi2020deep,liu2021orthogonal,miyato2018spectral}.
 In \citet{gresele2021independent} it  is argued based on the independent mechanism principle that the closely related function class of orthogonal coordinate transformations is a natural function class for representation learning.
Essentially, the argument relies on the well-known fact that two isotropic random vectors are very close to orthogonal up to an  exponentially small probability.
In Appendix~\ref{app:random_functions} we show a construction of a family of random functions that becomes increasingly isometric as $D\to \infty$. 
Let us finally remark that the investigation of the inductive bias of VAEs 
in \citet{rolinek,zietlow,embrace} revealed that their  loss function and architecture promote that the encoder implements an orthogonal coordinate transformation. The empirical success of VAEs on disentanglement tasks is therefore a further motivation to study theoretical properties of 
orthogonal coordinate transformations and local isometries.

For ICA with locally isometric mixing function the following identifiability result was shown.
\begin{theorem}[Theorem~1, \citet{horan2021when}]\label{th:isometry}
We assume that $\mc{P}$ satisfies Assumption~\ref{as:P2}.
Suppose $X=f(S)$ where $S\sim\P\in \mc{P}$ and $f\in \mc{F}_\iso$.
If $X\indistribution \tilde{f}(\tilde{S})$ for some $\tilde{f}\in \mc{F}_\iso$ and $\tilde{S}\sim\tilde{\P}\in\mc{P}$, then 
$f=\tilde{f}\circ P$ for some linear map $P$ which is a product of a permutation and reflections.
\end{theorem}
Note that \citet{horan2021when} only claimed disentanglement, but the stronger version stated presently follows as sketched below and is also a special case of Theorem~2 of \citet{buchholz2022function}.
The proof of this result proceeds in two steps. First, we can identify 
$f$ up to an orthogonal linear transformation $A$ by general results. 
Indeed, $\tilde{f}^{-1}\circ f:\Omega\to \R^d$ is a local isometry and
local isometries from $\Omega\subset \R^d\to\R^d$ are affine.
This is a result first shown by Liouville in the 19th century, a simple proof can be found in the recent review of \citet{hyvarinen2023identifiability}.	
A slightly different viewpoint showing the same result was given by \citet{Tenenbaum2000}.
After identifying $S$ up to linear transformations, we can as a second step apply the standard identifiability result for linear ICA  to identify $A$ up to permutations and reflection (there is no scale ambiguity here because we consider local isometries).

We emphasize that this separation into two steps, where first
identifiability up to linear transformations and then full identifiability are shown, is very common, e.g.,
a similar proof strategy has been applied for polynomial mixing functions \citep{ahuja2022towards}, piecewise linear functions \citep{kivva2022}, or for general mixing functions and interventional data \citep{buchholz2023learning}. 

\section{Robustness and approximate Identifiabilility}\label{sec:robustness}

\begin{figure*}[ht!]
\begin{center}
\includegraphics[width=\linewidth]{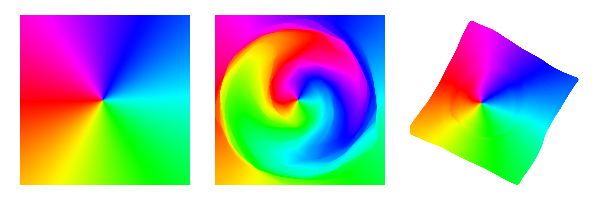}
%/Users/sbuchholz/PycharmProjects/perturbed_ica/plots/distances.pdf}
\end{center}
%\vspace{-.3cm}
\caption{
(Left) Color map of Gaussian latent variable $Z$, (Center) Color map of the transformed data $X=f(Z)$ where $f$ is a piecewise linear approximation of a radius dependent rotation (i.e., $f(Z)\overset{\mathcal{D}}{\approx} Z$), (Right) Representation $X'=f'(Z)$ learned by a VAE with ReLU activation functions initialized with $f$ and $f^{-1}$ for decoder and encoder 
and small variance.}
\label{fig:0}
\end{figure*}

Theorem~\ref{th:isometry} proves identifiability, i.e., all data representations agree up to permutation and reflections.
However, this uses crucially that the assumptions are exactly satisfied, i.e., $f$ is a true local isometry, the coordinates of $S$ are perfectly independent, and we know the distribution of $X$ exactly.

In real-world settings, typically none of these assumptions will hold exactly but at best approximately ('all models are wrong' \citep{box1976science}). Thus, identifiability results can only be relevant for applications when they are robust, i.e., they continue to hold in some approximate sense when the assumptions are mildly violated.
For, e.g., supervised learning or distribution learning, there is a large body of research (essentially the field of learning theory) addressing in particular the dependence on the sample size but also 
misspecification attracted considerable attention, e.g., \citet{white1980using,white1982maximum}. For  representation learning there are much fewer results. The sample complexity of linear ICA algorithms has been studied, e.g., by \citet{yau1992compact,tichavsky2006performance,wei2015convergence}. Moreover, \citet{zhang2008minimal} investigated ICA for a perturbed linear model, although mostly empirically and without stringent quantitative results.
To the best of our knowledge, no robustness results are known
for non-linear mixing functions $f$.  
The main focus of this work is to show a first robustness result for a misspecification of the function class, while we do not 
consider misspecification of independence of the ground truth samples and
finite sample effects here.

Let us emphasize that it is an important question because not all identifiability results have a meaningful extension to slight misspecification.  
As a simple motivating example, consider analytic functions (i.e., functions that can be expressed as a power series) and smooth functions (i.e., functions that have arbitrary many derivatives). While those function classes are superficially similar, they exhibit profound differences: Analytic functions $f:\R\to \R$ have the property that there is a unique
continuation of $f|_U$ from any arbitrarily small open set $U$ to a maximal domain of definition (this can be seen as arbitrary far o.o.d.\ generalization) while for smooth functions virtually no information about $f|_{U^\complement}$ can be deduced from $f|_U$.
This indicates that minor misspecification (analytic vs.\ smooth) can render identifiability results void, while the representation capacity 
remains almost the same and the difference can hardly be detected from data.
To connect this more closely to identifiability in representation learning, consider, e.g., an identifiability result such that 
the following properties hold (a similar discussion can be found in the Appendix of \citet{buchholz2023learning})
%\vspace{-.2cm}
\begin{itemize}%[noitemsep]
\item When assuming  that the mixing $f$ is in some  function class 
$\mc{F}$ then $f$ is identifiable (up to certain symmetries, e.g., permutations or linear maps)
\item When the mixing  is just assumed to be an injective and continuous function $f$ is not identifiable, i.e., for  $f\in \mc{F}$ there is $f'$ giving rise to the same observational distribution (for suitable latent distributions) but $f$ and $f'$ are entirely different data representations that are not related by the allowed symmetries.
\item The function space $\mc{F}$ is dense in the space of injective continuous functions, i.e., every injective and continuous function $f'$ can be approximated to arbitrary precision by $f\in \mc{F}$.
\end{itemize}
Note that the last property implies that for $f'\notin \mc{F}$ no meaningful distance to $\mc{F}$ can be defined, as it can be approximated arbitrarily well by functions in $\mc{F}$. This also makes the identifiability result brittle: While $f\in\mc{F}$ can be identified, there is a sequence of functions $\mc{F}\ni f_i\to f'$ converging to the spurious solution $f'$. Thus, it can be arbitrarily difficult to decide in practice whether $f$ generated the data or an 
approximation $f_i\in \mc{F}$ of the spurious solution $f'$.
 Concrete examples of this setting include
polynomial mixing functions and piecewise linear mixing functions
which can be used to identify, e.g., Gaussian or rectangular base distributions \citep{ahuja2022towards,kivva2022}.
We emphasize that these results nevertheless are important theoretical contributions, however, these identifiability results might be difficult to exploit in practice.
Our  observations  here are in spirit similar to a no free lunch theorem \citep{wolpert1997nofree} because it clarifies the fundamental tradeoff between
representation capacity of the function class and identifiability.
Note  that the lack of identifiability can also harm downstream performance, as shown by \citet{saengkyongam2023identifying}.

Let us try to make the point above concrete in a toy setting. 
We consider two-dimensional Gaussian latent variables $Z$.  Let $f$ be 
a piecewise linear mixing function $f$ 
that approximates a radius dependent rotation $m$ \citep{hyvarinen1999nonlinear} which is a map
such that $m(Z)\indistribution Z$. Consider observations $X=f(Z)\overset{\mc{D}}{\approx}Z$.
Note that $f$ and the identity map generate similar observations so it is hard to distinguish between the mixing function $f$ and the identity $\mathrm{id}$
even though $f$ is identifiable for Gaussian latent variables up to linear transformations (see \citet{kivva2022}). 
However, if we train a 
 VAE with ReLU activations that is initialized with $f$ and $f^{-1}$ as decoder and encoder, i.e., we initialize it with the true mixing,  it nevertheless fails to preserve this representation.  Instead, the VAE learns
a simple orthogonal transformation of the data, i.e., even though the map $f$ is in theory identifiable the learning signal is not sufficient to overcome the inductive bias of the learning method (here the  VAE) towards simple functions that, e.g., preserve simple geometric properties. 
An 
illustration of the learned mixing can be found in Figure~\ref{fig:0}.
Additional details for the setting can be found in Appendix~\ref{app:illustration}.

The main purpose of this work is to show that the function class $\mc{F}_\iso$ 
behaves differently, i.e., robust representation learning results can be proved.
To formalize this, we first need to introduce a way to measure the distance $\sodist(f)$
of a function $f$  to the space of local isometries $\mc{F}_\iso$.
This quantity should vanish ($\sodist(f)=0$) when $f\in \mc{F}_\iso$ 
and our goal is to show that when $\sodist(f)$  is small we can 
still approximately identify $f$. For functions $f:\Omega\subset \R^d\to\R^d$ we use 
the following quantity
\begin{align}\label{eq:dist_f_iso}
\begin{split}
\sodist_p^p(f, \Omega) = 
\int_\Omega  &\mathrm{dist}(D f(s),\mathrm{SO}(d))^p
\\
&+ \mathrm{dist}\big( (D f)^{-1}(s),\mathrm{SO}(d)\big)^p\,\d s.
\end{split}
\end{align} 
Here $\dist^2(A, \mathrm{SO}(d))=\min_{Q\in \SO(d)}|Q-A|^2$ refers to the euclidean distance to the space 
$\SO(d)=\{A\in \R^{d\times d}:\, AA^\top=\id, \, \det A=1\}$ of all rotations. We discuss the definition and properties of the distance  in more detail in Appendix~\ref{app:distances}.
The distance $\sodist$ (which is not a distance in the mathematical sense) measures how close a function $f$ is to a local isometry pointwise and integrates this.
If $f\in \mc{F}_\iso(\Omega)$ is a local isometry and if $\det Df>0$ everywhere (which is no loss of generality by reflecting one coordinate), then indeed 
$\sodist_p(f,\Omega) =0$ because $Df\in \SO(d)$ pointwise. For our results we need the second term in
\eqref{eq:dist_f_iso}, i.e., we measure the local deviation from being an isometry of the map and its inverse.
One could also define the metric by integrating with respect to the  latent distribution $\P$ instead of the Lebesgue measure, but we think that the latter is slightly more natural because a common assumption based on the independent mechanism assumption \citep{janzing2010causal,scholkopf2012causal,janzing2016algorithmic} is that $\P$ and $f$ are sampled independently. This also simplifies our analysis slightly.

Our main interest concerns high dimensional embeddings, i.e., 
mixing function $f:\R^d\to M\subset \R^D$ for $d\ll D$ where $M$ is 
an embedded submanifold. In this case, the definition 
 \eqref{eq:dist_f_iso} needs to be slightly adapted, by essentially 
restricting the target of $Df$ to the tangent space of $M$.
For $D\geq d$ we define
\begin{align}\label{eq:dist_f_iso_general}
\begin{split}
\sodist_p^p(f, \Omega) &= 
\int_\Omega \mathrm{dist}^p(D f(z),\SO(d, T_{f(z)}M))
\\
&\quad
+ \mathrm{dist}^p\big( (D f)^{-1}(z),\SO(T_{f(z)}M,d)\big)\,\d z
% \\
% &=
% \int_\Omega \mathrm{dist}(D f(z),\mathrm{SO}(d, T_{f(z)}M))^p
% + \mathrm{dist}\big( D f^{-1}(f(z)),\mathrm{SO}(T_{f(z)}M,d)\big)^p\,\d z
\end{split}
\end{align} 
where $T_{f(z)}M$ denotes the tangent space of $M$
at $f(z)$ and $\mathrm{SO}(d, T_{f(z)}M)$ denotes
the set of orthogonal matrices $Q\in \R^{D\times d}$
(i.e., $Q^\top Q=\id_d$) with range in $T_{f(z)}M$
that are orientation preserving (here we fix any orientation on $M$). To interpret the second summand, we need to associate certain matrix representations to the maps $(Df)^{-1}$ and 
$\SO(T_{f(z)}M,d)$ and we refer to Appendix~\ref{app:undercomplete} for a careful definition.

Let us now discuss and motivate the choice of distance in a bit more detail.
First, we emphasize that \eqref{eq:dist_f_iso} does not (directly) quantify how
well $f$ is globally approximated by a local isometry, i.e., 
whether there exists a $g\in \mc{F}_\iso$
such that, e.g., $\lVert f-g\rVert_1$ is small.
We note that from a modelling standpoint, it seems more reasonable to assume that 
\eqref{eq:dist_f_iso} is small for the ground truth mixing function than
assuming that $f$ can be globally well approximated by a local isometry because the former corresponds to the common assumption that $f$ preserves the structure of the data locally (in particular distances)  while it is not clear why $f$ should be  globally constrained. In Appendix~\ref{app:random_functions} we show that for  $d\ll D$ 
and certain classes of random functions $\sodist$ will be indeed small. This quantifies the observation that those random functions
become isometric as $D\to \infty$ (similar results can be found in \citet{reizinger2023independent}). 

%A surprising result discovered in the context of material science 
%is that 
%a bound on \eqref{eq:dist_f_iso} allows to conclude that $f$ can be globally well approximated by a local isometry. We will discuss this 
%in the next section.

%Finally we emphasize that 
%this ability to quantify distances from $\mc{F}_\iso$ and extend identifiability results to slightly misspecified settings is not
%obvious and for many other function classes for which identifiability results were shown this does not seem to be the case.
%E.g., polynomial function or piecewise linear functions it is not clear how to define such a distance as they are dense in all continuous functions. This is some indication that the non-identifiability of nonlinear-ICA prevents robustness for such $\mc{F}$.

%For real world settings it is typically more reasonable to assume that 
%$f$ is not contained $\mc{F}$ but only  close to $\mc{F}$ in some sense.
%Thus it is natural to investigate whether we can retain approximate 	
%identifiability in such a case. Note that we for $\mc{F}_\iso$ a suitable closeness measure exists. On the other hand, 
%for polynomial or piecewise linear mixings this might not be more involved because they are dense in the space of all continuous functions.

One difficulty in the setting is that as soon as we allow misspecification 
of 
$f$ the mixing function $f$ will typically be no longer identifiable
(otherwise the existing identifiability result could be strengthened to a larger function class). We focus mostly on the setting of nonlinear ICA and in this case, 
it is well known that $f$ is not identifiable.
Thus, we can only hope to prove a version of approximate identifiability 
that gracefully recovers the classical identifiability result as the misspecification disappears. 
%This creates additional technical difficulties, which we now address.
%First, we have to define a notion of approximate identifiability.
We here  define approximate identifiability of the latent variables based  on the mean correlation coefficient (MCC) 
which has been used before to evaluate the empirical performance of
representation learning algorithms.
For a pair of $d$ dimensional random variables $(S,\tilde{S})$ the MCC is defined
by
\begin{align}\label{eq:MCC_def}
\mathrm{MCC}(S,\tilde{S})=\max_{\pi \in S_d} d^{-1}\sum_{i=1}^d  |\rho(S_i, \tilde{S}_{\pi(i)})|
\end{align}
where $\rho(X,Y)=\Cov(X,Y)/ (\Var(X) \Var(Y))^{1/2}$ denotes the correlation coefficient. Note that  
$\mathrm{MCC}(S,\tilde{S})=1$ implies that $S_i=\lambda_i \tilde{S}_{\pi(i)}$, for a vector $\lambda\in \R^d$, i.e., we recover $Z$ up to permutation and scale.
For MCC values close to 1 this is approximately true, in particular 
$\tilde{S}=PS+\eps$ holds for a scaled permutation $P$ and some error $\eps$ that vanishes as the MCC goes to 1.
 More generally, one could also allow coordinatewise reparametrizations
of $\tilde{S}$
when those are not identifiable
\cite{gresele2021independent}, but this is not necessary when considering 
(perturbed) local isometries.
% To summarize, the main goal  this work we study how we can approximately 
% identify (in the sense of MCC close to 1) 
% We emphasize that we provide no completely 
% general definition of robustness for representation learning,
% and we do 

Let us elaborate on the difference between proper identifiability \citep{xi2023ideterminacy}
and approximate identifiability.
Recall that a proper identifiability results is of the form $f_\ast \P=\tilde{f}_\ast\tilde{\P}$ for $f,\tilde{f}\in \mc{F}$ and
$\P,\tilde{\P}\in \mc{P}$ implies that $\tilde{f}^{-1}\circ f$ is a tolerable ambiguity, e.g., a combination of permutation and reflection in Theorem~\ref{th:isometry}. In this case, an equivalent statement  is that $f_\ast \P=\tilde{f}_\ast\tilde{\P}$ implies
$\mcc(\tilde{f}^{-1}\circ f(S),S)=1$ for $S\sim \P$.
In contrast, for approximate identifiability results
we cannot simply take any $\tilde{f}$ such that  
$\tilde{f}_\ast\tilde{\P}=f_\ast \P$ because there will typically exist
arbitrarily chaotic spurious solutions with no guarantee on
$\mcc(\tilde{f}^{-1}\circ f(S),S)$ for $S\sim \P$.
Instead, we need to make  a specific choice for the  function $\tilde{f}$ 
such that $\tilde{f}_\ast\tilde{\P}=f_\ast \P$ 
which ensures that $\mcc(\tilde{f}^{-1}\circ f(S),S)\approx 1$ for a suitable class of $f$. In our case those will be functions such that $\sodist_p(f,\Omega)$ is small. 
Of course, our choice of $\tilde{f}$ is only allowed to depend on 
the observations $X=f(S)$ but not on $f$ itself.
Here the basic idea is to choose $\tilde{f}$ as close as possible to being a local isometry, i.e., we roughly chose $\tilde{f}$ such that
$\sodist(\tilde{f},\tilde{\Omega})$ is minimal under the constraint 
$\tilde{f}_\ast\tilde{\P}=f_\ast \P$. 
Let us now provide a very informal version of our final approximate identifiability result.
\begin{theorem}[Informal sketch]
Suppose $S\sim \P$ where $\P$ has independent components, and we observe
$X=f(S)$ for some mixing function $f$. Then we can find $\tilde{f}$
such that $\hat{S}=\tilde{f}^{-1}(X)$ satisfies for some $C>0$, $p>1$
\begin{align}
\mcc(\hat{S},S)\geq 1-C \sodist_p^2(f)
\end{align}
where $C$ depends on everything except $f$.
\end{theorem}
Note that if $f$ is far away from a local isometry, i.e., $\sodist_p(f)$
is large then the right-hand side is negative, and the 
obtained recovery guarantee is void which is not surprising as nonlinear ICA is not identifiable.
The actual statement can be found in Theorem~\ref{th:local_ica_compact}.
We emphasize again that robustness is an aspect of representation learning
that has previously attracted little attention, and we 
study one specific example, namely robustness of learning approximate local isometries. However, 
there are many alternative robustness questions and possible misspecifications, e.g, we could consider 
latent sources that are only approximately independent or settings where MCC might not be the right measure of identifiability. 
 
%
%will take the form that 
%for some given observational distribution 
%$f_\ast \P$ we can construct a function $\tilde{f}$ such that 
%$\mcc((\tilde{f}^{-1}f)_\ast (S), S) $ is close to 1 f, i.e., we need 
%an algorithm or an implicit definition of $\tilde{f}$. Those results 
%should recover the identifiability result as the misspecification disappears.
% Since we only consider approximate identifiability there is some degree of freedom, e.g., which of the conditions  

\section{Approximate linear Identifiability for approximate local Isometries}
\label{sec:approximate_isometries}
We now extend linear identifiability for locally isometric embeddings to approximate isometric embeddings. Restricting the ambiguity to a linear 
transformation is already an important first step which can 
be in principle combined with any result on causal representation learning for
linear mixing functions, examples of recent results for linear mixing functions
include \citet{ahuja2022interventional,seigal2022linear,varici2023score}.
The key ingredient for our results is the  rigidity statement
Theorem~\ref{th:rigidity} in Appendix~\ref{app:approximate_isometries}
which played an important role in the mathematical analysis of elastic materials (see 
\citet{ciarlet1997mathematical} for an overview).

The main consequence of this theorem that is relevant for our work here is the bound
\begin{align}
\begin{split}
    	\min_L&\lVert u-L\rVert_{L^q(\Omega)}
	\\
 &\leq 
	 C(\Omega, p)
\left(\int_\Omega  \dist(Du(s), \SO(d))^p\, \d s\right)^{\frac{1}{p}}
\end{split}
\end{align}
for all $u$ and $q=pd/(d-p)$ where the minimum is over functions of the form $L(s)=As+b$
with $A\in \SO(d)$.
This result shows that when a map $u$ has gradient pointwise close to any, potentially varying, rotation, it is globally close to an affine map. We emphasize that this is highly non-trivial, naively one would expect that the rotation could change 
with $s$. Note that when the right hand-side is zero, i.e., $f$ is a local isometry then we conclude that $u$ is affine recovering the fact that local isometries from $\R^d$ to $\R^d$ are already affine which is the key observation underlying Theorem~\ref{th:isometry}. This rigidity property of almost isometries renders them almost identifiable. This  then generalizes Theorem~\ref{th:isometry}. 

% Using this result, we can generalize Theorem~\ref{th:isometry} by deriving approximation results for the learning 
% of approximate isometries. 
%This already shows that for mixing function $f:\R^d\to \R^d$ the considered function class is not particularly expressive because it essentially contains slightly perturbed linear function and therefore cannot overcome the restrictions of linear maps. However, already local isometries constitute a rich class
%of embeddings
For simplicity,  we mostly present our results for the case where $D=d$
%, i.e., the mixing maps $\R^d\to\R^d$ 
in the main paper. The  additional technicalities  for $d<D$ (which is the main interest of our results) will be discussed in Appendix~\ref{app:undercomplete}. 
For a measure $\P\in \mc{P}$ with support $\Omega\subset\R^d$ we introduce the set of models
%TODO Needs to be fixed later when considering bounded support
\begin{align}\label{eq:def_of_M}
\begin{split}
%\mc{M}(f_\ast\P)=\{(g,\Q, \Omega'): g\in \mc{F}(\Omega'), \,\Q\in \mc{P},
%\, \Omega'\subset \R^d, \text{ where $g_\ast \Q = f_\ast \P$, $\mathrm{supp}(\Q)=\Omega'$}\}
\mc{M}(f_\ast\P)=\{(g,\Q, \Omega'): g\in \mc{F}(\Omega'), \text{ where}
\\
\text{ $g_\ast \Q = f_\ast \P$, $\mathrm{supp}(\Q)=\Omega'$}\}
\end{split}
\end{align}
that generate the observed distribution of $X=f(S)$. 
We now consider for a fixed $p$ the triple
\begin{align}\label{eq:def_of_g}
\begin{split}
&(g,\Q, \Omega')\in 
\\
&\argmin_{(\bar{g},\bar{Q},\bar{\Omega})\in \mc{M}(f_\ast\P)} \int_{\bar{\Omega}} \mathrm{dist}((D \bar{g})^{-1}(g(s)),\mathrm{SO}(d))^p\,\bar{\Q}(\d s),
\end{split}
\end{align}
i.e., we pick $(g,\Q, \Omega')$ such that $X\sim g_\ast \Q $ so that it generates the observational distribution and in addition make its inverse as isometric as possible.
For $D>d$ we replace the distance to $\SO(d)$ by the distance
to $\SO(T_{g(s)}M, d)$ (see \eqref{eq:def_of_g2} in Appendix~\ref{app:approximate_isometries} for details).
We emphasize that $\mc{M}(f_\ast\P)$ and therefore $g$ only depends on the observational distribution $f_\ast \P$ but not on $f$ directly. Note that we could also minimize a variant of $\Theta(g, \Omega')$
where we integrate with respect to $\Q$.
We remark in passing that we do not prove the existence of a minimizer   of \eqref{eq:def_of_g} (although this should be possible using the direct method of the calculus of variations).  Since the functional is lower bounded,  a minimizing sequence $g_n$ exists
and we can instead approximate the infimum up to arbitrarily small $\eps$
adding an arbitrarily  small additional error term to Theorem~\ref{th:almost_orthogonal} below (see Appendix~\ref{app:approximate_isometries} for details).
%In Appendix~\ref{app:undercomplete} we provide the slight modification required for the undercomplete case.
%Then the following result holds.

\begin{theorem}\label{th:almost_orthogonal}
Suppose we have a latent distribution $\P\in \mc{P}$ 
satisfying Assumption~\ref{as:P1}
with support $\Omega\subset \R^d$ where $\Omega$ is a bounded connected Lipschitz domain.
The observational distribution is given by $X=f(S)\in M\subset \R^D$
where $S\sim \P$ and $f\in \mc{F}(\Omega)$.
Fix a $1<p<\infty$.
Define $(g,\Q, \Omega')$ as in \eqref{eq:def_of_g} (as in 
\eqref{eq:def_of_g2} in
Appendix~\ref{app:undercomplete} in the undercomplete case).
Then there is $A\in \SO(d)$ such that $g^{-1}\circ f(s)= As + h(s)+b$  and $h$ satisfies for $p<d$ and $q=pd/(d-p)$ the bound 
\begin{align}\label{eq:bound_h}
\lVert h\rVert_{\P,q} \leq C_1 \sodist_p(f, \Omega).
\end{align} 
Here $C_1$ is a constant depending on  $d$, $p$, $\Omega$,
and the lower and upper bound on the density of $\P$.
\end{theorem}
The proof of this result, including the extension to the undercomplete case can be found in Appendix~\ref{app:approximate_isometries}. 
Let us continue our discussion on the meaning of approximate identifiability in the context of this theorem. First, we note that if $\sodist_p(f,\Omega)$ 
is not small we obtain no useful statement, except that our transformed data
is some function of the original data. This is not surprising: We cannot hope to recover any mixing function because the problem of learning $f$ from $f_\ast\P$ is not identifiable (even when assuming $\P$ known). 
Moreover, our statement only applies to one specific unmixing $g$ which is again unavoidable for the same reason. What we show is that if $f$ and $g$ are both close to being locally isometric, then the concatenation $g^{-1}f$ will be close to a linear (even orthogonal) map. Note that $g$ does not appear on the right-hand side of \eqref{eq:bound_h} because we choose $g$ to be the maximally isometric representation of our observations, but we know that this representation is more isometric than any alternative representation, in particular more isometric than $f$.
This allows us to bound the non-linearity $h$ in terms of $f$ only.
While this result does not have the simplicity of a standard identifiability
result it provides a more general viewpoint. Indeed, if 
$\sodist_p(f,\Omega)=0$, i.e., when $f$ is a local isometry we have $h=0$
and we recover $f$ up to a linear transformation
which is the standard identifiability result for local isometries (see Theorem~2 of \citet{horan2021when}). Our result extends this gracefully to functions that are approximate local isometries in the sense that $\sodist_p(f,\Omega)$ is small.
% we think that its claim that learning isometric representation will recover the source 
% This result indicates that we obtain approximate identifiability up to linear transformations up to a small error that decays gracefully with our measure $\sodist$ of deviation from being isometric. 
Let us add some remarks about this result.
\begin{remark}\label{rmk:1}
\begin{itemize}
%\item The extension to the undercomplete case $f:\R^d\to M\subset\R^D$ is mostly notational because we have to deal with the fact that $Df$ then is a linear map to the tangent space we need to define isommetries to this subspace.
% \item When $f$ is a (oriented) local isometry, i.e., $\Theta_p(f,\Omega)=0$ we find that we can identify $f$ up to orthogonal transformations, this recovers the identifiability result of Theorem~2 in \citep{horan2021when}.
% \item We do not need to explicitly bound the density of $\Q$ from above or below, it is sufficient to assume the existence of such a bound for $\P$.
\item The optimization problem for $g$ in \eqref{eq:def_of_g}
is non-convex, and  difficult to optimize in practice.
\item For $D=d$ the introduction of $g$ is not necessary, instead we can directly apply Theorem~\ref{th:rigidity} to  $f$
and just work with the original data $X$ and directly apply Theorem~\ref{th:rigidity} to $X=f(S)$. 
\item  The assumptions that $\Omega$ is connected and that the density of $\P$ is lower bounded are necessary.
In particular the result does not apply to distributions
$\P$ with disconnected support.

%TODO 
 % \item We do not know whether the  second term in the definition 
 % of the distance in \eqref{eq:dist_f_iso} is necessary (for $d<D$). However,  the expression on the right-hand side of \eqref{eq:def_of_g} might be unbounded for $f$ even  when the forward distance $\int_\Omega \dist(Df(s), \mathrm{SO}(d))^p\,\d s$ is arbitrarily small, so it cannot trivially be excluded. 
 \item There are alternative assumptions that allow us to remove (or bound) the second term in \eqref{eq:dist_f_iso}, e.g., assuming $Df^\top Df >c_1>0$, i.e., the smallest singular value of $Df$ is bounded below is sufficient (see Lemma~\ref{le:inverse_dist}). 
\end{itemize}
\end{remark}

%We introduce the shorthand 
%\begin{align}
%\sodist(f,\Omega, \SO(d))^p= \int_\Omega \mathrm{dist}(D f(z),\mathrm{SO}(d))^p
%+ \mathrm{dist}\big( (D f)^{-1}(f(z)),\mathrm{SO}(d)\big)^p\,\d z.
%\end{align}

\section{Perturbed linear ICA}\label{sec:perturbed}

In this section, we consider the problem of independent component analysis
where the mixing is a slight perturbation of a linear function. This is a problem of general interest beyond the main setting considered in this paper because in typical real-world applications 
the mixing will only be approximately linear, so understanding the effect of the non-linear part
is important. 
Concretely, we assume that
data is generated by a perturbed linear model
\begin{align}\label{eq:main_model}
x=f(s)= As+\eta h(s)
\end{align}
where $h:\R^d\to \R^d$ is a non-linear function and $\eta\in\R$ is a
small constant. We can assume that $h$ is centered, i.e., 
$\E(h(S))=0$ and that
 $A$ is 
the regression matrix when regressing $X$ on $S$.
This is equivalent to the relation $\mathbb{E}_S(S h(S)^\top)=0_{d\times d}$.

\begin{remark}\label{rem:h_norm}
We decided to express the nonlinear part of the model as $\eta h$ which allows us to consider $\eta\to0$ to shrink the size of the perturbation. As we check carefully in the proofs, all bounds only depend on $h$ through its norm $\lVert h\rVert_{\P,q}$
for some $q$ specified below.
\end{remark}

It is clear that $f$ cannot be identified from the distribution of $X$ as the mixing is nonlinear. 
Instead, we investigate how well  the linear part given by the matrix $A$ can be recovered
and to what degree we can recover the ground truth sources $S$.
Our results show that as $\eta\to 0$  we recover the identifiability results for linear ICA. One issue that creates substantial theoretical problems in this work 
and in general is that there is a gap between the statistical and computational aspects of ICA, in the sense that linear ICA is identifiable for non-Gaussian 
latent variables, however ICA algorithms typically require a stronger non-degeneracy condition on the latent distribution than non-Gaussianity of the latent variables. In the misspecified setting 
we need to define a (not necessarily computationally feasible) algorithm  to pick a mixing function for the observed data, which then results in the same limitations as the conventional algorithms face.

Let us now quickly review how the independent components can be identified.
Most ICA algorithms consider a function $H:S^{d-1}\to\R$ defined by
\begin{align}\label{eq:def_H}
H(w)=\E G\left(w^\top \Sigma_X^{-\frac12} X\right)
\end{align}
where $\Sigma_X$ denotes the covariance matrix of $X$ so that
$\Sigma_X^{-\frac12} X$ is whitened and $G$ is the so-called contrast function.
Then, under a suitable degeneracy assumption, the independent components correspond to an orthogonal  set of local extrema of $H$.

Indeed, as a motivation for the nonlinear case  we discuss in Appendix~\ref{app:proofs_ica}
the well-known calculation that for linear mixing functions (i.e., $\eta=0$) $H$ has a local extremum if $w^\top \Sigma_X^{-\frac12} X=S_i$ for some $i$, i.e., when $A^\top \Sigma_X^{-\frac12}w=e_i$. 
Note that for linear mixing $A$ the relation $\Sigma_X=AA^\top$ holds, and the
vectors
\begin{align}\label{eq:def_barw}
\bar{w_i}=  (AA^\top)^{\frac12}A^{-\top}e_i
\end{align}
thus satisfy $\bar{w}_i^\top \Sigma_{X}^{-\frac12}X=S_i$
and $|\bar{w}_i|=1$.
We also consider the  matrix $\bar{W}=A^{-1}(AA^\top)^{\frac12}$ which has rows $\bar{w}_i=\bar{W}^\top e_i$ and satisfies $\bar{W}\Sigma_{X}^{-\frac12}X=S$ for $\eta=0$. 

The main goal of this section is to show that this general picture remains approximately true in the perturbed setting, i.e., 
under minor regularity assumptions on $G$ there is a matrix $W$ close to $\bar{W}$ such that its rows are local extrema of $H$.
%We define the function $H:S^{d-1}\to \R$ by
%\begin{align}
%H(w)=\E G\left(\frac{wx}{\sigma_w}\right)
%\end{align}
%where $\sigma_w = \sqrt{\E((wX)^2)}$ is the normalization.
Let us collect the necessary assumptions for our results. 
% We need some bounds for the contrast function $G$.
\begin{assumption}\label{as:ICA1}
The function $G$ is even, three times differentiable, and there are 
constants $C_g$ and $d_g$ such that for $k\leq 3$
\begin{align}\label{eq:boundg}
|G^{(k)}(x)|\leq C_g (1+|x|)^{\max(d_g-k, 0)}
\end{align}
where $G^{(k)}$ denotes the $k$-th derivative of $G$.
\end{assumption}
We will write $g=G'$ from now on following the notation in the field.
%A slightly stronger bound is necessary to prove convexity of $H$.
%\begin{assumption}\label{as:ICA2}
%The bound \eqref{eq:boundg} holds in addition for $G^{(3)}=g''$.
%%The function $G$ is even, three times differentiable, and satisfies
%%for some constants $C_g$ and $d_g$
%%\begin{align}\label{eq:boundgpp}
%% G(x),g(x),g'(x),g''(x)\leq C_g (1+|x|)^{d_g}
%%\end{align}
%%where $g=G'$.
%\end{assumption}
Note that commonly used contrast functions like $G=\ln\circ\cosh$ or $G=|\cdot|^4$ satisfy this assumption.
We need some regularity assumption on the source variables.
\begin{assumption}\label{as:ICA3}
Write $q=\max(d_g,3)$.
Assume that the latent sources $S$  satisfy
for some constant $M<\infty$
\begin{align}
\E(|S|^{q})=M
\end{align}
\end{assumption}
%\vspace{-.2cm}
Finally, we need a condition that ensures that the contrast function
can single out the latent variable $S_i$ which is also necessary in the linear case.
\begin{assumption}\label{as:ICA4}
The latent variables $S$ satisfy for an $\alpha_i\neq 0$
\begin{align}\label{eq:cond_nondegen}
\E(S_ig(S_i)-g'(S_i)) = \alpha_i.
\end{align}
\end{assumption}
%\vspace{-.2cm}
Then we have the following result.
\begin{theorem}\label{th:pert}
Suppose that $X$ and $S$ are random variables satisfying the relation~\eqref{eq:main_model} and the distribution $\P$ of $S$ 
has independent components (see Assumption~\ref{as:P2})
and $\E(S)=0$ and $\E(S_i^2)=1$.
Assume  $G$ is an even 
contrast function that satisfies the pointwise bounds in Assumption~\ref{as:ICA1} and that $S$ and $G$ satisfy  Assumptions~\ref{as:ICA3} and \ref{as:ICA4} for some $1\leq i\leq d$. Assume that $\lVert h\rVert_{\P,q}\leq 1$ 
where $q=\max(d_g,3)$ is defined in Assumption~\ref{as:ICA3}.
 Let $\bar{w}_i$ be defined as in \eqref{eq:def_barw}.
Then, there is an $\eta_0>0$  (depending on $d$, $\alpha_d$, $A$, $M$, $C_g$, and $q$ but not on $h$) such that for $\eta<\eta_0$ there is a local extremum $w_{i}$ of $H$
close to $\bar{w}_i$ satisfying $|\bar{w}_i-w_{i}|=O(\eta)$. 
In addition, $H$ is strictly convex or concave  (depending on the sign of $\alpha_i$) in a neighborhood  $B_{\kappa}(\bar{w}_i)$ with $\kappa$ independent of $\eta<\eta_0$.
\end{theorem}
\begin{remark}
In fact, we can characterize $w_{d}$ a bit more precisely by
\begin{align}\label{eq:tilde_wd}
w_{d}=\Sigma_X^{\frac12}A^{-\top}\begin{pmatrix}
\eta\alpha_d^{-1}v \\ 1 
\end{pmatrix}+O(\eta^{2})=\tilde{w}_d+O(\eta^{2})
\end{align}
where $v_i = \E((A^{-1}h)_d(S)g'(S_d)S_i)
+ \E((A^{-1}h)_i g(S_d))$ for $i<d$ and $\alpha_d=\E(S_dg(S_d)-g'(S_d)) $ as in \eqref{eq:cond_nondegen} and similar expressions hold for 
$i<d$.
\end{remark}
The proof can be found in Appendix~\ref{app:proofs_ica}.
We can apply the result to all coordinates simultaneously and obtain the following slight generalization.
\begin{theorem}\label{th:pert_matrix}
Under the same assumptions as in Theorem~\ref{th:pert} where we assume
that the Assumption~\ref{as:ICA4} holds for all $S_i$ with $1\leq i\leq d$ there is a matrix $W$ whose rows $w_i=W^\top e_i$ satisfy $|w_i|=1$,
$w_i$ is a local extremum of $H$,  $A^\top \Sigma_X^{-\frac12}w_i=e_i
+O(\eta)$, and $|W-\bar{W}|=O(\eta)$.
\end{theorem}
The main message of these results is that all standard ICA algorithms based on a contrast function $G$ can extract $w_i$ approximately in this setting, and the resulting matrix $W$ is close to $\bar{W}$ the ground truth unmixing of the linear part.  Let us clarify this through the example of gradient-based algorithms.
\begin{corollary}\label{co:convergence}
Under the same assumptions as in Theorem~\ref{th:pert} and for $\eta<\eta_0$
gradient ascent (if $\alpha_i > 0)$ or descent (if $\alpha_i < 0$)
with sufficiently small step size will converge to $w_i$ locally.
\end{corollary}
This result is a direct consequence of Theorem~\ref{th:pert} and convergence of gradient descent  for strictly convex functions. We provide some more details in Appendix~\ref{app:proofs_ica}.  
The results above established that we can approximately find $A$
as $\eta\to 0$. Next we show that we do not only approximately recover the 
linear mixing, but we can also approximately recover the sources $S$.
\begin{theorem}\label{th:mcc1}
Consider the same assumptions as in Theorem~\ref{th:pert_matrix} and 
let $W$ be (for $\eta<\eta_0$) the matrix as in the statement of Theorem~\ref{th:pert_matrix}.
Define $\hat{S}=W\Sigma_X^{-\frac12}X$. Then there is a constant $C_2$ depending 
on $d$, $\alpha_i$, $A$, $M$, $C_g$, and $q$  such that 
\begin{align}
\mcc(\hat{S},S)\geq 1-C_2\eta^2.
\end{align}
\end{theorem}
The proof of this result can be found in Appendix~\ref{app:proofs_ica}. It is based on a general bound for the $\mcc$ stated in  Lemma~\ref{le:mcc2}
in Appendix~\ref{app:mcc}.
We provide some experimental evidence that the derived bounds have the optimal scaling in $\eta$ in Appendix~\ref{sec:experiment}.

\begin{remark}\label{rmk:unsatisfactory}
The results might appear a bit unsatisfactory because we only show existence of local extrema of $H$ that allow us to recover the sources approximately 
and local convergence of gradient descent or ascent to these extrema.
However, we do not claim that these are the only extrema or there is an efficient algorithm to find all extrema in polynomial time. This problem is, however, not a feature of the perturbed setting, but the same issue arises already for linear ICA \citep{hyvarinen1999fast}. Note that  for the special case of the kurtosis-based contrast function 
$G(s)=s^4-3$ with linear mixing it can be shown (under additional assumptions) that the vectors $\bar{w}_i$ are the only local maxima (or minima) of $H$
and they can be found in polynomial time as shown by \citet{arora2012provable}.
We expect that this result generalizes to our perturbed setting.
\end{remark}

Finally, we combine the two parts of the analysis to show approximate identifiability of ICA with approximately locally isometric mixing.
To simplify the analysis we assume that the support of $\P$ is compact.
Then we can combine our previous results to obtain the following theorem.
\begin{theorem}\label{th:local_ica_compact}
    Assume that $S\sim \P$ with bounded and connected support $\Omega$ satisfies Assumptions~\ref{as:P1} and \ref{as:P2}.
    Suppose observations are given by $X=f(S)\in M\subset \R^D$ for some $f\in \mc{F}(\Omega)$. 
    Assume that $\P$ satisfies Assumption~\ref{as:ICA1}, \ref{as:ICA3}, and
    \ref{as:ICA4} for some contrast function $G$. Let $q=\max(d_g, 3)$
    and $p=dq/(d+q)$.
    Define $(g,\Q,\Omega')$ as in 
    \eqref{eq:def_of_g} (for $d=D$ and as in 
    \eqref{eq:def_of_g2} in Appendix~\ref{app:undercomplete} for $d<D$) and shift $g$ such that
    $\tilde{X}=g(X)$ is centered. Define $\hat{S}=W\Sigma_{\tilde{X}
    }^{-\frac12}\tilde{X}$
    as in Theorem~\ref{th:pert_matrix} for observations $\tilde{X}$. 
    Then the following bound holds for $\sodist_p(f, \Omega)$
    sufficiently small and some $C_3>0$ 
    \begin{align}
        \mcc(\hat{S},S)\geq 1 - C_3\sodist_p^2(f, \Omega).
    \end{align}
    Here $C_3$ depends on $C_1$
    and $C_2$ from Theorem~\ref{th:almost_orthogonal}
    and \ref{th:mcc1} and the bounds on the density of $\P$.
\end{theorem}
%\vspace{-.2cm}
The informal version of this theorem is that by learning a maximally locally isometric transformation of our data followed by running a linear ICA-algorithm we can approximately recover the true sources 
if the ground truth mixing is close to a local isometry. 
A proof of this theorem can be found in Appendix~\ref{app:approx_iso_ica}.
Let us add a few remarks about this result.
%\vspace{-.2cm}
\begin{itemize}%[noitemsep]
    \item For local isometries ($\sodist_p(f,\Omega)=0$) we essentially recover
    Theorem~2 in \cite{horan2021when}.
    \item The assumption that $\P$ has bounded support combined with the independence assumption implies that $\P$ is supported on a cuboid. 
    It was shown in \citep{ahuja2022interventional,Rothetal23} that then linear transformations 
    of $\P$ are identifiable (up to scaling and permutation) and we expect
    that this can be generalized to perturbed linear maps, but it is non-trivial to prove this and to obtain a quantitative statement as we obtain above.
    \item In Appendix~\ref{app:approx_iso_ica} we also show how to relax the condition that $\P$ has compact support, which comes at a price of additional technical difficulties.
\end{itemize}

\section{Conclusion}\label{sec:conclusion}
In this paper, we initialized the theoretical investigation of robustness 
results for representation learning with misspecified models.
While we showed robustness results for linear ICA and for almost isometric\
embeddings 
there are several closely related questions, of which we name a few.
First, it would be of interest to show for the perturbed linear ICA model finite sample guarantees, local convergence guarantees  for algorithms not relying on gradient descent, e.g., for the fast-ICA algorithm, and analyze the effects of additive noise. 
Secondly, one limitation of this work is that we assume that the observational distribution $f_\ast \P$ 
and the manifold it is supported on
are exactly known in the undercomplete case. Note that
by the Nash embedding theorem \citep{nash1954c1} we can approximate uniformly any submanifold by a sequence of local isometries. Thus, here additional regularization by, e.g., penalizing higher order derivatives is necessary to extend robustness results to, e.g., finite sample size or noisy settings. 
The  broader questions that this work motivates is what class  of
identifiability results are brittle and which are robust to misspecification (or finite sample bias).
 A deeper understanding of this landscape (and its intersection 
with real-world datasets) 
 will help to understand which identifiability results give rise to
useful learning signals, and thus provide some guidance for the  
  development of novel algorithms for (causal) representation learning.

 \section*{Acknowledgements}
 This work was supported by the Tübingen AI Center.

\section*{Impact statement}
This paper presents work whose goal is to advance the field of Machine Learning. There are many potential societal consequences of our work, none which we feel must be specifically highlighted here.
%\newpage

%\bibliography{./../../ml.bib}

\bibliography{ml}
\bibliographystyle{icml2024}
%%%%%%%%%%%%%%%%%%%%%%%%%%%%%%%%%%%%%%%%%%%%%%%%%%%%%%%%%%%%

\appendix
\newpage
\onecolumn

\begin{center}
\Large{\textbf{Supplementary Material}}
\end{center}
Here we collect additional results, details, and proofs for the main paper. This appendix is structured as follows. We first introduce some notation and definitions used throughout the paper in Appendix~\ref{app:notation}
then we
 state and prove a lemma that can be used to bound the MCC in our setting
in Appendix~\ref{app:mcc}. We
gather some elementary results about the distance of matrices to the orthogonal group in Appendix~\ref{app:distances}, and  we extend these results to mappings to submanifolds  $M\subset \R^D$
and define our measure of non-isometry also in  this setting
in Appendix~\ref{app:undercomplete}.
In Appendix~\ref{app:approximate_isometries}
we prove our results on approximate linear identifiability for approximate local isometries, followed by our
careful analysis of perturbed linear ICA in Appendix~\ref{app:proofs_ica}. The proof of our combination of the two results can be found in Appendix~\ref{app:approx_iso_ica}.
In Appendix~\ref{app:illustration}
we provide the missing details for the experimental illustration of non-robustness and in Appendix~\ref{sec:experiment} we 
confirm empirically the theoretical convergence rates which we derived for perturbed linear ICA.
Finally, in Appendix~\ref{app:random_functions} we show that a certain class of random functions becomes increasingly isometric as the ambient dimension grows.

%\section{Auxiliary technical results}\label{app:auxiliary_results}
%In this section we collect technical results that we need in our proofs. 

\section{Notation and Mathematical Definitions}
\label{app:notation}
In the paper we use notation from different fields. Let us collect a few definitions and references to the literature.
\paragraph{Linear Algebra} We denote by $|A|=|A|_F$ the Frobenius norm on matrices, i.e., $|A|^2=\tr AA^\top=\sum_{i,j}|A_{ij}|^2$.
We will frequently use unitary invariance 
$|X|_F=|QX|_F$ for $Q\in \SO(d)$ of the Frobenius norm.
The symbol $A^{-\top}=(A^{-1})^\top$ denotes the inverse of the transpose.
\paragraph{Probability Theory} We denote by $f_\ast \P$ the pushforward 
of a measure $\P$ along $f$. It is defined by 
$f_\ast\P(A)=\P(f^{-1}A)$ and the distribution of $f(Z)$ is $f_\ast \P$ if
$Z\sim \P$. We denote by $\lVert f\rVert_{\P,p}=\left(\E_\P |f||^p\right)^{1/p}$
the $L^p$ norm.

\paragraph{Sobolev spaces} 
For differentiable functions $u:\Omega\to \R$ we define the Sobolev norm
\begin{align}
\lVert u\rVert_{W^{1,p}(\Omega)}^p=
\int_\Omega |u|^p(s)+|Du|^p(s) \, \d s.
\end{align}
The Sobolev embedding $W^{1,p}(\Omega)\hookrightarrow L^q(\Omega)$ is continuous
for $q=pd/(d-p)$. For details and a proper definition of Sobolev spaces we refer to \citet{adams2003sobolev}. We denote by $\lVert f\rVert_{p}$
the usual $L^p$ norm with respect to the Lebesgue measure.

\paragraph{Differential Geometry} We consider mixing functions $f:\R^d\to M\subset \R^D$ where $M$ is an embedded submanifold. In this case, we need some 
notions from differential geometry such as the tangent space $T_pM$ whose definitions can be found in, e.g., \citet{lee2013introduction}. 

\section{Bounding $\mcc$ for almost Permutations}\label{app:mcc}
Here, we provide a Lemma that allows us to control the MCC between two data representations when the transformation is perturbed linear and the linear part is close to a permutation matrix.
\begin{lemma}\label{le:mcc}
Let $\P$ be a centered probability measure. Assume that for $Z\sim \P$ the bound
\begin{align}\label{eq:mcc_ass1}
\max_{i, j} \frac{\lVert Z_i\rVert_{\P,2}}{\lVert Z_j\rVert_{\P,2}}
\leq c_1
\end{align}
holds for some constant $c_1$.
Let $T(z)= Az + h(z)$ where $A$ is a matrix such that
there is a permutation $\rho$ and a constant $c_2$ with 
\begin{align}\label{eq:mcc_ass2}
\max_i \frac{1}{|A_{i,\rho(i)}|}\sum_{j\neq \rho(i)} |A_{i,j}|\leq c_2
\end{align}
and 
$h$ has the property that for some constant $c_3$
\begin{align}\label{eq:mcc_ass3}
\frac{\lVert h_{i} \rVert_{\P,2} }{|A_{i,\rho(i)}|\lVert Z_{\rho(i)}\rVert_{\P,2}}\leq c_3.
\end{align}
Then the bound 
\begin{align}
\mcc(T(Z), Z) \geq 1-2c_3-2c_1c_2
\end{align}
holds.
\end{lemma}
%\begin{lemma}\label{le:mcc}
%Let $\P\in \mc{P}$ with support $\Omega$ where $\mc{P}$ satisfies Assumptions~\ref{as:1} and
%\ref{as:2}. Assume $\P$ has the property
%\begin{align}\label{eq:mcc_ass1}
%\max_{i, j} \frac{\lVert Z_i\rVert_{\P,2}}{\lVert Z_j\rVert_{\P,2}}
%\leq c_1
%\end{align}
%for some constant $c_1$.
%Let $T(z)= Az + h(z)$ where $A$ is a matrix such that
%there is a permutation $\rho$ and a constant $c_2$ with 
%\begin{align}\label{eq:mcc_ass2}
%\max_i \frac{1}{|A_{i,\rho(i)}|}\sum_{j\neq \rho(i)} |A_{i,j}|< c_2
%\end{align}
%and 
%$h$ has the property that for some constant $c_3$ such that
%\begin{align}\label{eq:mcc_ass3}
%\frac{\lVert h_{i} \rVert_{\P,2} }{|A_{i,\rho(i)}|\lVert Z_{\rho(i)}\rVert_{\P,2}}\leq c_3.
%\end{align}
%Then the bound 
%\begin{align}
%\mcc(T_\ast \P, \P) \geq 1-2c_3-2c_1c_2
%\end{align}
%holds.
%\end{lemma}
\begin{proof}
We lower bound the expression \eqref{eq:MCC_def} for the permutation $\pi=\rho^{-1}$ in the statement of the lemma. Let $Z\sim \P$ and $\tilde{Z}=T(Z)$.
Consider the correlation coefficients
\begin{align}
\rho(Z_{\rho(i)}, \hat{Z}_{i})
=\frac{\Cov (Z_{\rho(i)}, \hat{Z}_{i})}{\sigma_{Z_{\rho(i)}}\sigma_{\hat{Z}_{i}}}.
\end{align}
We find, using that $Z_{\rho(i)}$ is centered

\begin{align}
\begin{split}\label{eq:cov_bound}
\left|\Cov (Z_{\rho(i)}, \hat{Z}_{i}) - A_{i,\rho(i)} \E(Z_{\rho(i)}^2)\right|
&=\left|\E(Z_{\rho(i)} \hat{Z}_{i}) - A_{i,\rho(i)} \E(Z_{\rho(i)}^2)\right|
\\
&= \left|\sum_{j} A_{i,j}\E(Z_{\rho(i)} Z_j) +\E(Z_{\rho(i)} h_i(Z) ) - A_{i,\rho(i)} \E(Z_{\rho(i)}^2)\right|
\\
&=
\left|\sum_{j\neq \rho(i)} A_{i,j}\E(Z_{\rho(i)} Z_j) +\E(Z_{\rho(i)} h_i(Z) )\right|
\\
&\leq \lVert Z_{\rho(i)}\rVert_{\P,2}\left(\lVert h_{i}(Z)\rVert_{\P,2}
+ \sum_{j\neq \rho(i)} |A_{i,j}|\lVert Z_i\rVert_{\P,2}\right)
\\
&\leq 
|A_{i,\rho(i)}| \lVert Z_{\rho(i)}\rVert_{\P,2}^2
(c_3+c_2c_1).
\end{split}
\end{align}
Note that the first terms vanish for independent latent variables.
Next we control the variance of $\hat{Z}_i$. We find 
using the triangle inequality
\begin{align}
\begin{split}\label{eq:sigma_bound}
	\sigma_{\hat{Z}_i}^2
	&\leq 	\E (\hat{Z}_i^2)
	\leq  \left( \lVert h_i(Z)\rVert_{\P,2}+\sum_{j} |A_{ij}| \lVert Z_j\rVert_{\P,2} \right)^2
	\\
	&\leq 	
	\lVert Z_{\rho(i)}\rVert_{\P,2}^2
	\left(c_3 |A_{i,\rho(i)}| +|A_{i,\rho(i)}|+c_1 \sum_j |A_{i,j}|\right)^2
	\\
	&\leq 
\lVert Z_{\rho(i)}\rVert_{\P,2}^2 |A_{i,\rho(i)}|^2
\left(1 + c_3+ 	c_2c_1\right)^2.
	\end{split}
\end{align}
We now combine  \eqref{eq:cov_bound} and \eqref{eq:sigma_bound} with $(1+x)^{-1}\geq 1-x$ and find
\begin{align}
|\rho(Z_{\rho(i)}, \hat{Z}_{i})|
\geq \frac{|A_{i,\rho(i)}| \lVert Z_{\rho(i)}\rVert_{\P,2}^2 (1 - c_3 - c_1 c_2)}{\lVert Z_{\rho(i)}\rVert_{\P,2}\cdot \lVert Z_{\rho(i)}\rVert_{\P,2} |A_{i,\rho(i)}|
\left(1 + c_3+ 	c_2c_1\right)}
\geq 
(1 - c_3 - c_1c_2)^2\geq 1-2c_3-2c_1c_2.
\end{align}
The claim follows by bounding
\begin{align}
\mcc(T(Z) ,Z) \geq d^{-1} \sum_i |\rho(Z_{\rho(i)}, \hat{Z}_{i})|
\geq 1-2c_3-2c_1c_2.
\end{align}

\end{proof}
The previous lemma has minimal assumptions on $\P$ and $h$. We now show that 
the lemma can be improved when assuming independence of the coordinates of $\P$.
Otherwise, the statement and proof are similar.
\begin{lemma}\label{le:mcc2}
Let $\P$ be a centered probability measure on $\R^d$ satisfying Assumption~\ref{as:P2}, i.e., with independent components. 
 Assume that for $Z\sim \P$ the bound
\begin{align}\label{eq:mcc2_ass1}
\max_{i, j} \frac{\lVert Z_i\rVert_{\P,2}}{\lVert Z_j\rVert_{\P,2}}
\leq c_1
\end{align}
holds for some constant $c_1$.
Suppose that  $T(z)= Az + h(z)$  and $\E_\P(Z_ih_j(Z))=0$ for all $1\leq i,j\leq d$. 
Assume that $A$ is a matrix such that
there is a permutation $\rho$ and a constant $c_2$ with 
\begin{align}\label{eq:mcc2_ass2}
\max_i \frac{1}{|A_{i,\rho(i)}|^2}\sum_{j\neq \rho(i)} |A_{i,j}|^2\leq c_2^2
\end{align}
and 
$h$ has the property that for some constant $c_3$
\begin{align}\label{eq:mcc2_ass3}
\frac{\lVert h_{i} \rVert_{\P,2} }{|A_{i,\rho(i)}|\lVert Z_{\rho(i)}\rVert_{\P,2}}\leq c_3.
\end{align}
Then the bound 
\begin{align}
\mcc(T(Z), Z) \geq 1-\frac{c_3^2}{2}-\frac{(c_1c_2)^2}{2}
\end{align}
holds.
\end{lemma}
\begin{proof}
Again, we lower bound the expression \eqref{eq:MCC_def} for the permutation $\pi=\rho^{-1}$ in the statement of the lemma. Let $Z\sim \P$ and $\tilde{Z}=T(Z)$.
We first control
\begin{align}
\rho(Z_{\rho(i)}, \hat{Z}_{i})
=\frac{\Cov (Z_{\rho(i)}, \hat{Z}_{i})}{\sigma_{Z_{\rho(i)}}\sigma_{\hat{Z}_{i}}}.
\end{align}
Using that $Z_{\rho(i)}$ is centered followed by $\E(Z_iZ_j)=\E(Z_ih_j(Z))=0$
(by independence and the assumption for $h$) we get
\begin{align}
\begin{split}\label{eq:cov_bound2}
\Cov (Z_{\rho(i)}, \hat{Z}_{i}) 
&=\E(Z_{\rho(i)} \hat{Z}_{i}) 
\\
&= \sum_{j} A_{i,j}\E(Z_{\rho(i)} Z_j) +\E(Z_{\rho(i)} h_i(Z) ) 
\\
&=
 A_{i,\rho(i)}\E(Z_{\rho(i)}^2 ) .
\end{split}
\end{align}
Next we control the variance of $\hat{Z}_i$. We expand the product and find
(exploiting again independence and that $h$ and $Z$ are uncorrelated)
\begin{align}
\begin{split}\label{eq:sigma_bound2}
	\sigma_{\hat{Z}_i}^2
	&= 	\E (\hat{Z}_i^2)
	=\E \left(  h_i(Z)+\sum_{j} A_{i,j}  Z_j\right)^2
	\\
	&=
	\lVert h_i\rVert_{\P,2}^2+\sum_{j\neq \rho(i)}|A_{i,j}|^2\lVert Z_j\rVert_{\P,2}^2 + |A_{i,\rho(i)}|^2\lVert Z_{\rho(i)}\rVert_{\P,2}^2
	\\
	&\leq 	
	 |A_{i,\rho(i)}|^2 \cdot \lVert Z_{\rho(i)}\rVert_{\P,2}^2
	\left(1 + c_3^2 +c_1^2\frac{1}{|A_{i,\rho(i)}|^2} \sum_j |A_{i,j}|^2\right)
	\\
	&\leq 
\lVert Z_{\rho(i)}\rVert_{\P,2}^2 |A_{i,\rho(i)}|^2
\left(1 + c_3^2+ 	(c_2c_1)^2\right).
	\end{split}
\end{align}
Combining again \eqref{eq:cov_bound2} and \eqref{eq:sigma_bound2}
with $\sqrt{1+x}\leq 1+x/2$ (for $x>0$) and $(1+x)^{-1}\geq 1-x$, we find
\begin{align}
|\rho(Z_{\rho(i)}, \hat{Z}_{i})|
\geq \frac{|A_{i,\rho(i)}| \cdot  \lVert Z_{\rho(i)}\rVert_{\P,2}^2}{\lVert Z_{\rho(i)}\rVert_{\P,2}\cdot |A_{i,\rho(i)}|\cdot \lVert Z_{\rho(i)}\rVert_{\P,2}
\sqrt{1 + c_3^2+ 	(c_2c_1)^2}}
\geq 
\left(1 + \frac{c_3^2}{2}+ 	\frac{(c_2c_1)^2}{2}\right)^{-1}\geq 1 - \frac{c_3^2}{2} - 	\frac{(c_2c_1)^2}{2}.
\end{align}
We finish by noting
\begin{align}
\mcc(T(Z) ,Z) \geq d^{-1} \sum_i |\rho(Z_{\rho(i)}, \hat{Z}_{i})|
\geq 1 - \frac{c_3^2}{2}- 	\frac{(c_2c_1)^2}{2}.
\end{align}

\end{proof}

\section{Auxiliary result on distances to the orthogonal group }
\label{app:distances}
%%TODO add back
%In this section we collect some (rather elementary) results about the orthogonal group.
%First, in Appendix~\ref{app:dist_sod} we investigate distances to the the orthogonal group which is the  metric that we use to  quantify deviations from being a local isometry. In Appendix~\ref{app:ortho_permutation} we 
%show that certain orthogonal matrices are close to permutation matrices.
%\subsection{Elementary results for the distance to $\SO(d)$}
%\label{app:dist_sod}
%%TODO  end add back
We need some elementary and well-known facts on the distance of matrices to the orthogonal group. 
Let $A\in \R^{d\times d}$ be a matrix with $\det A > 0$. Then we consider a Singular Value Decomposition (SVD)
$A= U\Sigma V^\top$ where $U, V \in \SO(d)$ and $\Sigma=\diag(\sigma_1,\ldots, \sigma_d)\in \diag(d)$ contains the singular values $\sigma_i>0$.
The following Lemma holds.
\begin{lemma}\label{le:sod}
For $A\in \R^{d\times d}$ with $\det A>0$ and any SVD $A=U\Sigma V^\top$  as above, the relation
\begin{align}
 \dist(A,\mathrm{SO}(d))^2 = |(A^\top A)^{\frac12}-\id|_F^2=
 | \Sigma - \id |_F^2 = \sum_{i=1}^d (\sigma_i-1)^2
\end{align}
holds, moreover the minimizer of $| A-Q|_F^2$ is given by $
Q=UV^\top$.
\end{lemma}
\begin{proof}
First, we note that the last three expressions agree.
For the last two this is obvious, and we find
\begin{align}
    \sqrt{A^\top A}= \sqrt{V\Sigma^2V^\top }=
    V \sqrt{\Sigma^2 }V^\top=V\Sigma V^\top.
\end{align}
Then we conclude by unitary invariance $|RX|_F=|X|_F$ for
$R\in \SO(d)$ of the Frobenius norm that
\begin{align}
| \sqrt{A^\top A}-\id|_F=|V\Sigma V^\top-\id|_F=|\Sigma-\id|_F
\end{align}
It remains to show that also the first expression agrees with the other three expressions.
We find for $Q\in \SO(d)$ using 
again unitary invariance that
\begin{align}\label{eq:distSO}
\begin{split}
| A- Q|_F^2
&=
| U\Sigma V^\top-Q|_F^2
=
| \Sigma- U^\top Q V|_F^2
\\
&= 
| \Sigma|_F^2
+ | U^\top Q V|_F^2
- 2 \tr \Sigma(U^\top Q V)
=
| \Sigma|_F^2
+d-2\sum_{i=1}^d \sigma_i (U^\top Q V)_{ii}.
\end{split}
\end{align}
Here we used $| R|_F^2=\tr RR^\top =d$ in the last step.
Since $\sigma_i>0$ and since $R_{ii}\leq 1$ for $R\in \SO(d)$ we conclude
\begin{align}
| A- Q|_F^2\geq 
| \Sigma|_F^2
+d-\sum_{i=1}^d \sigma_i 
\end{align}
with equality if $(U^\top Q V)_{ii}=1$ for $1\leq i\leq d$ and
thus $U^\top Q V=\id$ and therefore $Q=UV^\top$. From \eqref{eq:distSO}
we find 
\begin{align}
| A- UV^\top|_F^2 =| \Sigma|_F^2
+d-2\sum_{i=1}^d \sigma_i 
=\sum_{i=1}^d (\sigma_i^2 -2\sigma_i +1).
\end{align}
\end{proof}
This has the following consequence.
\begin{lemma}\label{le:inverse_dist}
Let $A\in \R^{d\times d}$ with $\det A>0$ and denote the smallest singular value of $A$ by $\sigma_{\min}$. Then 
\begin{align}
\dist(A^{-1}, \SO(d))^2\leq \frac{1}{\sigma_{\min}^2}
\dist(A, \SO(d))^2
\end{align}
\end{lemma}
\begin{proof}
Note that the singular values of 
$A^{-1}$ are $\sigma_i^{-1}$ 
where $\sigma_i$ denote the singular values of $A$.
Using Lemma~\ref{le:sod}  we then find
\begin{align}
 \dist(A^{-1},\mathrm{SO}(d))^2=\sum_{i=1}^d
 (\sigma_i^{-1}-1)^2
 =\sum_{i=1}^d \sigma_i^{-2}
 (\sigma_i-1)^2\leq \sigma_{\min}^{-2}\dist(A, \SO(d))^2.
\end{align}
\end{proof}
We need one further lemma that concerns the distance to $\SO(d)$ of a product of matrices.
\begin{lemma}\label{le:prod_sod}
Let $A,B\in \R^{d\times d}$ with $\det A>0$, $\det B>0$.
Then 
\begin{align}\label{eq:sodist_1}
\dist(AB, \SO(d))\leq
\frac32\dist(A, \SO(d))+
\frac32\dist(B, \SO(d)).
\end{align}
This implies for $p>1$
\begin{align}\label{eq:sodist_p}
\dist^p(AB, \SO(d))\leq
3^{p}\left(\dist^p(A, \SO(d))+
\dist^p(B, \SO(d))\right).
\end{align}
\end{lemma}
\begin{proof}
Denote the projections of $A$ and $B$ on $\SO(d)$ by $Q$ and $R$
respectively.
Then
\begin{align}
\begin{split}
\dist(AB, \SO(d))&\leq
| AB-QR|_F
\\
&=
| (A-Q)R + Q(B-R)+(A-Q)(B-R)|_F
\\
&\leq | (A-Q)R |_F+ | Q(B-R)|_F+|(A-Q)(B-R)|_F
\\
&\leq | A-Q|_F+ | B-R|_F+\sqrt{| A-Q|_F| B-R|_F}
\\
&\leq 
\frac32 | A-Q|_F+ \frac32 | B-R|_F
\\
&=\frac32 \dist(A, \SO(d))+
\frac32 \dist(B, \SO(d)).
\end{split}
\end{align}
Here we used again unitary equivalence of the Frobenius norm and  the Cauchy-Schwarz inequality.
This proves \eqref{eq:sodist_1} and observe that by the generalized mean inequality $(a+b)^p\leq 2^{p-1}(a^p+b^p)\leq 2^p(a^p+b^p)$
we obtain \eqref{eq:sodist_p}.
\end{proof}

\section{Extension to the undercomplete case}\label{app:undercomplete}
The main interest of representation learning concerns the case where we actually learn a representation of a low dimensional submanifold embedded in a high dimensional space. This creates some additional technical difficulties that we address here. 
We always assume that the mixing $f:\Omega\to M\subset\R^D$ is a diffeomorphism on an embedded submanifold.
We can identify the tangent space $T_{f(s)}M$ with an affine subspace of $\R^D$ 
and we center this subspace at $0$ instead of $f(s)$ to obtain a linear subspace.
We consider the standard Riemannian metric on $\R^d$ and $\R^D$ inducing 
also a metric on $T_{f(s)}M$.
 The differential $Df(s)$ defines a bijective linear map $\R^d\cong T_s\R^d\to T_{f(s)}M$.  
 
To simplify the notation and restrict attention to the essential requirements we consider a linear subspace $H\subset \R^D$ and we assume that 
$A\in \R^{D\times d}$ is a full rank matrix with $\mathrm{range}(A)=H$, in particular $\dim(H)=d$.
 We now characterize the matrix representations
 of maps $\R^d\to H\subset \R^D$.
 For this we consider the unique polar decomposition $A=UP$ where $U:\R^d\to \R^D$ with $U\in \SO(d,D)$ and $P:\R^d\to \R^d$ is symmetric and positive definite. Note that $P$ is given by the expression $P=\sqrt{A^\top A}$.
 In the following $U$ will always denote the orthogonal matrix from such a  polar decomposition.
 \begin{lemma}\label{le:polar}
 Let $A=UP$ be the polar decomposition of $A$. Then every linear map $T\in \R^{D\times d}$ with $\mathrm{Range}(T)\subset H=\mathrm{Range}(A)$
 can be written uniquely as $T=UB$ for some $B\in \R^{d\times d}$.
 \end{lemma}
 \begin{proof}
We first claim that   $UU^\top T=T$. 
 Let $x\in \R^d$ then $Tx=Uy$ for some $y\in \R^d$ because $T$ maps
 to $H$ and $U$ is surjective on $H$. We thus find
 \begin{align}
 UU^\top Tx = UU^\top Uy=Uy=Vx
 \end{align}
 where we used $U^\top U=\id_d$ (recall $U\in \SO(d,D)$). 
 Then we find $T=UB$ with $B=U^\top T\in \R^{d\times d}$. Uniqueness follows by injectivity of $U$.
 \end{proof}

We now consider the space of isometries $\SO(d,H)$ consisting of 
all linear maps $Q:\R^d\to \R^D$ such that $\mathrm{Range}(Q)\subset H$ and $|Qv|=|v|$ for all $v\in \R^d$. Note that  $\SO$ in addition requires preservation of orientation, but we can either just choose the orientation induced by $A$ or ignore this issue and think of orthogonal matrices $\mathrm{O}(d,H)$. 
We nevertheless write $\SO$ to be consistent with the literature.
Note that $\SO(d,H)\subset \SO(d,D)$. 
We can now define the distance 
\begin{align}
\dist(A, \SO(d, H)) = \min_{Q\in \SO(d, H)}|A-Q|_F.
\end{align}
This distance can be characterized as follows.
\begin{lemma}
Let $T\in \R^{D\times d}$ be a linear map with $\mathrm{Range}(T)\subset H$.
Then
\begin{align}
\dist(T, \SO(d, H))=\dist(B, \SO(d)).
\end{align}
where $B$ is the unique matrix such that $T=UB$ (see Lemma~\ref{le:polar}).
\end{lemma}
\begin{proof}
The crucial observation is that the relation 
\begin{align}\label{eq:relarion_so_d_h}
\SO(d, H)=U\cdot \SO(d)
\end{align}
holds where $U \cdot \SO(d)=\{UV: \, V\in \SO(d)\}$. 
 We first consider the inclusion $\subset$. Consider $V\in \SO(d, H)$.
 We have seen in the proof of Lemma~\ref{le:polar} that
 $V=U(U^\top V)$ so we only need to show that $U^\top V\in \SO(d)$. We find
\begin{align}
 (U^\top V)^\top U^\top V=V^\top UU^\top V= V^\top V=\id_d
\end{align} 
where we used $ UU^\top V=V$ and $V\in \SO(d,D)$.
Therefore $U^\top V\in \SO(d)$ and thus $V\in U\, \SO(d)$.
Now we consider the reverse inclusion $\supset$.
Let $Q\in \SO(d)$. Then we find
\begin{align}
\mathrm{Range}(UQ)=\mathrm{Range}(U)=\mathrm{Range}(UP)
=\mathrm{Range}(A)=H.
\end{align}

Moreover $|UQv|^2=v^\top Q^\top U^\top UQv=v^\top Q^\top Qv=|v|^2$ and therefore
$UQ\in \SO(d, H)$.
Thus we find 
\begin{align}
\dist(T,\SO(d, H))= 
\dist(UB, U \,\SO(d)) =\dist(B,\SO(d))
\end{align}
where the last step used the unitary invariance of the Frobenius norm, i.e.,
$|UB|=|B|$ for $U\in \SO(d,D)$ and $B\in \R^{d\times d}$.
\end{proof}
Recall that in the polar decomposition $A=UP$ the matrix $P$ is given by 
$P=(A^\top A)^\frac12$.
Then the lemma above implies that for $A$ such that $H=\mathrm{range}(A)$
\begin{align}\label{eq:dist_so_under}
    \dist^2(A,\SO(d, H))=\dist^2((A^\top A)^{\frac12},\SO(d))
    =\sum  (\sigma_i -1)^2
\end{align}
where $\sigma_i$ are the singular values of $A$ and we used Lemma~\ref{le:sod} in the last step.

We also need to consider  isometries mapping $H\to \R^d$
and the distance to these isometries.
The space of $\SO(H,d)$ consists of all distance preserving linear maps
$H\to \R^d$. 
Clearly, there is no unique matrix representation $\R^{d\times D}$ of such maps because the map is not defined on $H^\perp$, i.e., there are many extensions to $\R^D$.
We essentially consider the extension by zero to the orthogonal complement, i.e, for a map $T:H\to \R^d$ we consider the extension $T:\R^D\to \R^d$ (not reflected in the notation) such that
$Tv=0$ for any $v\perp H$, i.e., $v^\top w=0$ for all $w\in H$.
Using this extension, we can identify such maps $T$ with a unique
matrix $\R^{d\times D}$. Then we define
\begin{align}
    \dist(T, \SO(H,d))=\min_{Q\in \SO(H,d)} |T-Q|_F
\end{align}
where we identify $T$ and $Q$ with matrix representations as explained before.
We have the following simple lemma.
\begin{lemma}
Let $T:H\to \R^d$ be a linear map, which we identify with  a matrix representation $T\in \R^{D\times d}$ as explained above. Let $U:\R^d\to H\subset \R^D$ be an isometry. Then there is a matrix $B\in \R^{d\times d}$ such that $T=BU^\top$ and 
\begin{align}
    \dist(T, \SO(H,d))=\dist(B,\SO(d)).
\end{align}
    
\end{lemma}
\begin{proof}
Since $U$ is an isometry to $H$ we find $U\in \SO(d,D)$ and $U^\top$ maps $H$ isometrically to $\R^d$ (because $U^\top (U v) =\id_d v=v$).
Therefore, there is a matrix $B\in \R^{d\times d}$ such that
$Tv=BU^\top v$ for all $v\in H$. Let $w\in H^\perp$.
Then for all $v\in \R^d$ we have $Uv\in H$ and thus $0=\langle w, Uv\rangle=\langle U^\top w, v\rangle$ which implies $U^\top w=0$.
We conclude that $T=BU^\top$. 
Next we claim that similar to \eqref{eq:relarion_so_d_h}
the relation
\begin{align}\label{eq:relarion_so_h_d}
\SO(H, d)= \SO(d)\cdot U^\top
\end{align}
holds. The inclusion $\supset$ follows, since the concatenation of isometries is an isometry (and $U^\top H^\perp=\{0\}$.
The reverse inclusion can be shown by using that $Q\in \SO(H,d)$
can be expressed as $Q=BU^\top $ for some $B$ and noting that since $U^\top\in \SO(H, d)$ we find that $B$ must be an isometry on $\R^d$ and therefore $B\in \SO(d)$.

Plugging everything together, we find
\begin{align}
\begin{split}
     \dist^2(T, \SO(H,d))&=
     \min_{Q\in \SO(H,d)}
     |Q-T|^2_F
     =
     \min_{Q'\in \SO(d)} |BU^\top - Q' U^\top|_F^2
     =  \min_{Q'\in \SO(d)}\tr U(B-Q')^\top (B-Q')U^\top
    \\
    &=  \min_{Q'\in \SO(d)}\tr (B-Q')^\top (B-Q')U^\top U
     =  \min_{Q'\in \SO(d)}\tr (B-Q')^\top (B-Q')\id_d
     \\
     &=  \min_{Q'\in \SO(d)}|B-Q'|_F^2=\dist^2(B,\SO(d)).
     \end{split}
\end{align}
This ends the proof.
\end{proof}

% We identify  linear maps from $H$ to $\R^d$ with matrices
% of the form $B'U^\top\in \R^{d\times D}$ where as before $A=UP$ is the polar decomposition with $U\in \R^{D\times d}$ and $B'\in \R^{d\times d}$ is an arbitrary matrix. Note that there is a unique $B'$ for any such map because
% $U^\top$ defines a bijective isometric map from $H$ to $\R^d$ because $U^\top U=\id_d$. 
% Then we define for a map $H$ to $\R^d$ the distance
% \begin{align}
% \dist(B'U^\top, \SO(H, \R^d))=
% \dist(B', \SO(d)).
% \end{align}
Note that  an invertible matrix $B$ induces a bijective map from $\R^d$ to $H$ through the matrix representation $UB$. 
Its inverse has the  matrix representation
$B^{-1}U^\top$, indeed, then we find $B^{-1}U^\top UB=\id_d$.
Now we can define the extension of the distance $\sodist$ to the undercomplete case. We define for $f:\Omega\subset \R^d\to M\subset \R^D$
\begin{align}\label{eq:dist_f_iso_general2}
\begin{split}
\sodist_p^p(f, \Omega) &= 
\int_\Omega \mathrm{dist}^p(D f(z),\mathrm{SO}(d, T_{f(z)}M))
+ \mathrm{dist}^p\big( (D f)^{-1}(z),\mathrm{SO}(T_{f(z)}M,d)\big)\,\d z
% \\
% &=
% \int_\Omega \mathrm{dist}(D f(z),\mathrm{SO}(d, T_{f(z)}M))^p
% + \mathrm{dist}\big( D f^{-1}(f(z)),\mathrm{SO}(T_{f(z)}M,d)\big)^p\,\d z
\end{split}
\end{align} 
where $(Df)^{-1}(z)$ denotes the inverse of the map $Df(z)$ viewed as a linear map with target $T_{f(z)}M$. Note that for $d=D$ this definition agrees with the definition in \eqref{eq:dist_f_iso}.

By our results so far, we can equivalently write
\begin{align}
\begin{split}\label{eq:dist_f_iso_general3}
\sodist_p^p(f, \Omega) &= 
\int_\Omega \mathrm{dist}^p(P(D f(s)),\mathrm{SO}(d))
+ \mathrm{dist}^p\big( P^{-1}(Df(s)),\mathrm{SO}(d)\big)\,\d s
\\
&=
\int_\Omega \mathrm{dist}^p((D f(s)^\top Df(s))^{\frac12},\mathrm{SO}(d))
+ \mathrm{dist}^p\big( (D f(s)^\top Df(s))^{-\frac12},\mathrm{SO}(d)\big)\,\d s
\end{split}
\end{align} 
where $P(Df(s))=\sqrt{Df(s)^\top Df(s)}$ denotes the unique matrix $P$ in the polar decomposition $Df(s)=UP$.

We need the following extension of Lemma~\ref{le:prod_sod}.
\begin{lemma}\label{le:prod_sod2}
Suppose $H=\mathrm{Range}(A)$ for a full rank matrix $A\in \R^{D\times d}$.
Let $T\in \R^{D\times d}$ with $\mathrm{range}(T)\subset H$ and $S:H\to \R^d$
a linear map  which we identify with its matrix representation
$S\in \R^{d\times D}$ as explained above. Then $ST$ is the matrix representation of the concatenation of the linear maps, and we assume that $\det(ST)>0$. Then
 the bound
\begin{align}
\dist(ST,\SO(d))\leq \frac32\dist(S, \SO(H,d))
+ \frac32\dist(T, \SO(d,H))
\end{align}
holds. For $p\geq 1$  this implies
\begin{align}
\dist^p(ST,\SO(d))\leq 3^{p}\dist^p(S, \SO(H,d))
+ 3^{p}\dist^p(T, \SO(d,H)).
\end{align}
\end{lemma}
\begin{proof}
Denote as before the polar decomposition of $A$ by $A=UP$.
We have observed that there are matrices $B$, $B'$
such that $S=B'U^\top$ and $T=UB$.
Then we have $ST=B'U^\top UB=B'B$. Moreover, we have seen
that 
\begin{align}
\dist(T, \SO(d,H))
= \dist(B , \SO(d)), \quad 
\dist(S, \SO(H, d))=\dist(B',\SO(d)).
\end{align}
Thus we find using Lemma~\ref{le:prod_sod} (for $A=B'$, $B=B$)
\begin{align}
\begin{split}
\dist(ST,\SO(d))
=\dist(B'B,\SO(d))&\leq \frac32\dist(B',\SO(d))
+\frac32\dist(B,\SO(d))
\\
&=\frac32\dist(T, \SO(d,H))+\frac32\dist(S, \SO(H, d)).
\end{split}
\end{align}

\end{proof}

%\section{Proofs for the results in Section~\ref{sec:approximate_isometries}}
\section{Proof and extension of approximate linear identifiability for
approximate isometries}
\label{app:approximate_isometries}
In this section we provide  the proofs for the results in Section~\ref{sec:approximate_isometries}. 
First, we state the key rigidity statement that we use in the proof of Theorem~\ref{th:almost_orthogonal}.
\begin{theorem}[Theorem~3.1 in \citet{friesecke2002rigidity}]
\label{th:rigidity}
Let $\Omega\subset \R^d$ be a bounded Lipschitz domain\footnote{A Lipschitz domain is slightly informally a set whose boundary $\partial \Omega$ can be expressed as the union of the graphs of Lipschitz functions (see \cite{adams2003sobolev} for a complete definition).}. Then there is for $p>1$ a constant $C(\Omega,p)$ such that for each $u\in W^{1,p}(\Omega, \R^d)$
there is a linear function $L(s)=As+b $ with $b\in \R^d$, $A\in \SO(d)$
such that
\begin{align}\label{eq:rigidity_bound}
\lVert u-L\rVert_{W^{1,p}(\Omega)}^p\leq C(\Omega, p) 
\int_\Omega  \dist(Du(s), \SO(d))^p\, \d s.
\end{align}
\end{theorem}
\begin{remark}
\begin{itemize}
\item The reference above only stated the case $p=2$ the simple extension to general $p$ can be found in Section~2.4 in \citet{conti2006rigidity}.
\item The main interest for the study of elasticity is the bound on the gradient 
because this is related to the energy. Here, we are only interested in the deviation $L-u$ (where actually simpler proofs of similar results are possible  \citep{kohn1982new}). Using for $p<d$ the Sobolev embedding $W^{1,p}(\Omega)\hookrightarrow
L^q(\Omega) $ for $q=pd/(d-p)$ the Theorem above implies
\begin{align}\label{eq:rigidity_bound_sobolev}
\begin{split}
	\lVert u-&L\rVert_{L^q(\Omega)}\leq C(\Omega, p)
 \lVert u-L\rVert_{W^{1,p}(\Omega)}
	\\
	&\leq 
	 C(\Omega, p)
\left(\int_\Omega  \dist(Du(s), \SO(d))^p\, \d s\right)^{\frac{1}{p}}.
\end{split}
\end{align} 
\end{itemize}
\end{remark}

Let us now discuss one lemma that we split off from the proof of Theorem~\ref{th:almost_orthogonal} because
we will use it in the proof of Theorem~\ref{th:iso_ica} below again.

\begin{lemma}\label{le:simple_bound_T2}
Suppose that $d\leq D$ and $f:\Omega\subset \R^d\to M\subset \R^D$ and $g:\Omega'\to M$
are diffeomorphisms,  $\P$ and $\Q$ are measures on $\Omega$ and $\Omega'$ such that $f_\ast \P=g_\ast \Q$. Suppose the density of $\P$ is lower bounded
on $\Omega$.
Let $T=g^{-1}\circ f$, then 
\begin{align}
\begin{split}\label{eq:bound_dT}
\int_\Omega  \mathrm{dist}^p&(D T(z),\mathrm{SO}(d))\, \d z
\\
&\leq 
3^{p}\int_\Omega  \mathrm{dist}^p(Df(s) ,\mathrm{SO}(d, T_{f(s)}M))\, \d s
+C\int_{\Omega'}  \mathrm{dist}^p(D g^{-1}(g(s)) ,\mathrm{SO}(T_{g(s)}M, d))
\, \Q(\d s)
\end{split}
\end{align}
where $C>0$ depends on the lower bound of the density and $p$.
\end{lemma}
\begin{remark}
    For $d=D$ the formula \eqref{eq:bound_dT} simplifies to
    \begin{align}
\begin{split}
\int_\Omega  \mathrm{dist}(D T(z),\mathrm{SO}(d))^p\, \d z
\leq 
3^{p}\int_\Omega  \mathrm{dist}^p(Df(s) ,\mathrm{SO}(d))\, \d s
+C\int_{\Omega'}  \mathrm{dist}^p(D g^{-1}(g(s)) ,\mathrm{SO}(d))
\, \Q(\d s).
\end{split}
\end{align}
\end{remark}
\begin{proof}
%We introduce the transition function $T=g^{-1} f:\Omega\to \Omega'$.
First we note that $g^{-1}:M\to \Omega$
is by assumption differentiable and by the chain rule we have
$Dg^{-1}(f(s)) Df(s)=DT(s)$. 
We now find using Lemma~\ref{le:prod_sod2} that 
\begin{align}
\begin{split}\label{eq:bound_t_sod3}
\int_\Omega  \mathrm{dist}^p&(D T(s),\mathrm{SO}(d))\, \d s
=
\int_\Omega  \mathrm{dist}^p(D g^{-1}(f(s)) Df (s),\mathrm{SO}(d))\, \d s
\\
&\leq 
3^{p}\int_\Omega  \mathrm{dist}^p(D g^{-1}(f(s)) ,\mathrm{SO}(T_{f(s)}M, d))
+\mathrm{dist}^p( Df (s),\mathrm{SO}(d, T_{f(s)}M))
\, \d s
\end{split}
\end{align}
Now we use the  transformation formula for 
push-forward measures, and the assumption $g_\ast \Q=f_\ast \P$
and the lower bound on the density of $\P$
\begin{align}
\begin{split}\label{eq:trafo_g_f}
\int_\Omega  \mathrm{dist}^p(D g^{-1}(f(s)) ,\mathrm{SO}(T_{f(s)}M,d))
\, \d s
&\leq C\int_\Omega  \mathrm{dist}^p(D g^{-1}(f(s)) ,\mathrm{SO}(T_{f(s)}M,d))
\, \P(ds)
\\
&= C\int_{f(\Omega)}  \mathrm{dist}^p(D g^{-1}(x) ,\mathrm{SO}(T_{x}M,d))
\, (f_\ast\P)(\d x)
\\
&=C\int_{g(\Omega')}  \mathrm{dist}^p(D g^{-1}(x) ,\mathrm{SO}(T_{x}M,d))
\, (g_\ast\Q)(\d x)
\\
&=C
\int_{\Omega'}  \mathrm{dist}^p(D g^{-1}(g(s)) ,\mathrm{SO}(T_{g(s)}M,d))
\, \Q(\d s).
\end{split}
\end{align}
The last two displays together imply the claim.
\end{proof}

We now provide the extensions of Theorem~\ref{th:almost_orthogonal} to  $d<D$. 
In this case, we can still define the set $\mc{M}(f_\ast\P)$ as in \eqref{eq:def_of_M}. We define $g$ by
\begin{align}\label{eq:def_of_g2}
(g,\Q, \Omega')\in \argmin_{(\bar{g},\bar{Q},\bar{\Omega})\in \mc{M}(f_\ast\P)} \int_{\bar{\Omega}} \mathrm{dist}(D \bar{g}^{-1}(x),\mathrm{SO}(T_xM, d))^p\,\bar{g}_\ast\bar{\Q}(\d x).
\end{align}
Here we need to integrate the deviation from 
an isometry over the observational distribution $\bar{g}_\ast\Q=f_\ast \P$.
Again, this agrees with the definition of \eqref{eq:def_of_g} for $d=D$.
To avoid the assumption that the minimum exists, we also let $(g_\eps, \Q,\Omega')\in \mc{M}(f_\ast \P)$  
be any function such that 
\begin{align}
\begin{split}\label{def:g_eps}
\int_{\bar{\Omega}} \mathrm{dist}(D g_\eps^{-1}(x),\mathrm{SO}(T_xM, d))^p\,(g_\eps)_\ast\bar{\Q}(\d x)
    \leq 
    \inf_{(\bar{g},\bar{Q},\bar{\Omega})\in \mc{M}(f_\ast\P)} \int_{\bar{\Omega}} \mathrm{dist}(D \bar{g}^{-1}(x),\mathrm{SO}(T_xM, d))^p\,\bar{g}_\ast\bar{\Q}(\d x)+\eps.
    \end{split}
\end{align}
% Lemma~\ref{le:simple_bound_T} has the following extension to this case.

Then we can state and prove the following complete version of Theorem~\ref{th:almost_orthogonal}.

\begin{theorem}\label{th:almost_orthogonal2}
Suppose we have a latent distribution $\P\in \mc{P}$ 
satisfying Assumptions~\ref{as:P1}
with support $\Omega\subset \R^d$ where $\Omega$ is a bounded connected Lipschitz domain.
The observational distribution is given by $X=f(Z)\in M\subset \R^D$
where $Z\sim \P$ and $f\in \mc{F}(\Omega)$.
Fix a $1<p<\infty$.
Let $(g_\eps,\Q, \Omega')\in \mc{M}(f_\ast\P)$ be any function satisfying \eqref{def:g_eps}.
Then there is $A\in \SO(d)$ and $b\in \R^d$ such that $g_\eps^{-1}\circ f(z)= Az + b + h(z)$  and $h$ satisfies the bound
\begin{align}\label{eq:bound_h2}
\lVert h\rVert_{\P, q} \leq C \sodist_p(f, \Omega)+C\eps^{\frac{1}{p}}
\end{align} 
for $q=pd/(d-p)$.
Here $C$ is a constant depending on  $d$, $p$, $\Omega$,
and the lower and upper bound on the density of $\P$.
\end{theorem}
It is clear that this theorem is more general than Theorem~\ref{th:almost_orthogonal}.
\begin{proof}
As before, we call the transition function $T=g^{-1} f:\Omega\to \Omega'$.
First, we observe that by definition of $g_\eps$
\begin{align}
    \begin{split}\label{eq:inf_eps}
        \int_{\bar{\Omega}} \mathrm{dist}(D g_\eps^{-1}(x),\mathrm{SO}(T_xM, d))^p\,(g_\eps)_\ast\bar{\Q}(\d x)
   & \leq 
    \inf_{(\bar{g},\bar{Q},\bar{\Omega})\in \mc{M}(f_\ast\P)} \int_{\bar{\Omega}} \mathrm{dist}(D \bar{g}^{-1}(x),\mathrm{SO}(T_xM, d))^p\,\bar{g}_\ast\bar{\Q}(\d x)+\eps
    \\
    &\leq 
    \int_{\Omega}  \mathrm{dist}^p(D f^{-1}(f(s)) ,\mathrm{SO}(T_{f(s)}M,d))
\, \P(\d s)+\eps.
    \end{split}
\end{align}
Here we used in the second step that  $(f,\P,\Omega)\in \mc{M}(f_\ast\P)$,
i.e., $f$ is a valid representation of our data so that this provides an upper bound on the infimum.
We now find using Lemma~\ref{le:simple_bound_T2} and \eqref{eq:inf_eps}
\begin{align}
\begin{split}\label{eq:bound_t_sod}
&\int_\Omega  \mathrm{dist}^p(D T(s),\mathrm{SO}(d))^p\, \d s
\\
&\leq 
3^{p}\int_\Omega  \mathrm{dist}^p(Df(s) ,\mathrm{SO}(d, T_{f(s)}M))\, \d s
+C\int_{\Omega'}  \mathrm{dist}^p(D g^{-1}(g(s)) ,\mathrm{SO}(T_{f(s)}M,d))
\, \Q(\d s)
\\
&\leq 
3^{p}\int_\Omega  \mathrm{dist}^p(Df(s) ,\mathrm{SO}(d, T_{f(s)}M))\, \d s
+C
\int_{\Omega}  \mathrm{dist}^p(D f^{-1}(f(s)) ,\mathrm{SO}(T_{f(s)}M,d))
\, \P(\d s)+C\eps
\\
&\leq
3^{p}\int_\Omega  \mathrm{dist}^p(Df(s) ,\mathrm{SO}(d, T_{f(s)}M))\, \d s
+C
\int_{\Omega}  \mathrm{dist}^p(D f^{-1}(f(s)) ,\mathrm{SO}(T_{f(s)}M,d))
\, \d s+C\eps
\\
&=C\sodist_p^p(f,\Omega)+C\eps.
\end{split}
\end{align}
Note that in the second to last step we used the upper bound on the density of $\P$.

 Now we apply Theorem~\ref{th:rigidity} (or rather its consequence
\eqref{eq:rigidity_bound_sobolev}  which states that there is 
a matrix $A\in \SO(d)$ and $b\in \R^d$ such that
\begin{align}
\left(\int_\Omega |T(s) - As-b|^q\, \d s \right)^{\frac{1}{q}}
\leq C(\Omega)\left(\int_\Omega  \mathrm{dist}^p(D T(s),\mathrm{SO}(d))\, \d s
\right)^{\frac{1}{p}}.
\end{align}
Using the upper  bound on the density of $\P$ and the last two displays we find
\begin{align}
\begin{split}
\left(\int_\Omega |T(s) - As-b|^q\, \P(\d s)\right)^{\frac{1}{q}}
&\leq C \left(\int_\Omega |T(s) - As-b|^q\, \d s\right)^{\frac{1}{q}}
\\
&\leq C\left(\int_\Omega  \mathrm{dist}^p(D T(s),\mathrm{SO}(d))\, \d s
\right)^{\frac{1}{p}}
\\
&\leq 
C\sodist_p(f,\Omega)+C\eps^\frac{1}{p}.
\end{split}
\end{align}
Here we used $(a+b)^{\frac{1}{p}}\leq (2a)^{\frac{1}{p}}+(2b)^{\frac{1}{p}}$
This completes the proof.
\end{proof}

\section{Proofs for the results on perturbed linear ICA}
\label{app:proofs_ica}
%\subsection{Proof of Theorem~\ref{th:pert}}
In this section we prove Theorem~\ref{th:pert}
and the extension in Theorem~\ref{th:mcc1}. Two technical difficulties when proving this result is that we consider a function defined on the sphere and that we need to whiten the data. Therefore,
we first prove a linearized result in Lemma~\ref{le:linearized} from which the Theorem can be  deduced after some technical algebraic manipulations. To motivate the calculations, 
let us first briefly sketch the well-known
unperturbed case of linear ICA \citep{hyvarinen2000independent}.
\paragraph{Proof sketch of the Linear Result}
We assume that $X$ is whitened and $X=AS$.
Let $w_0$ with $|w_0|=1$ be such that $w_0X=w_0AS=e_d S = S_d$.
Note that for $w$ in a neighborhood of $w_0$ we can write for some 
$\eps_i$ (small)
\begin{align}
wX = \sqrt{1-(\eps_1^2+\ldots + \eps_{d-1}^2)}
S_d+\eps_1 S_1+\ldots + \eps_{d-1} S_{d-1}.
\end{align}
Indeed, since $S_i$ are independent and with unit variance,
we find that the prefactor of $S_d$ has to be $\sqrt{1-(\eps_1^2+\ldots + \eps_{d-1}^2)}$ to ensure that $wX$ has unit variance.
Note that 
\begin{align}
1 -\sqrt{1-(\eps_1^2+\ldots + \eps_{d-1}^2)}
=\tfrac12 \sum_{i=1}^{d-1} \eps_i^2 + O(\eps^4)
\end{align} 
where $\eps^2=\sum_{i=1}^{d-1} \eps_i^2$ is the $l2$ norm.
Then we can Taylor expand (denoting $G'=g$)
\begin{align}
G(Xw)
=  G(S_d)+ \left(\sum_{i=1}^{d-1} \eps_i S_i-\frac12 S_d\sum_{i=1}^{d-1} \eps_i^2\right)g(S_d)
+\frac12 g'(S_1)\left(\sum_{i=1}^{d-1} \eps_i S_i\right)^2 + O(\eps^3).
\end{align}

Taking the expectation over this expression, we obtain
\begin{align}
\begin{split}
\E(G(Xw))&=\E(G(S_d))+\E\left( \left(\sum_{i=1}^{d-1} \eps_i S_i-\frac12 S_d\sum_{i=1}^{d-1}\eps_i^2\right)g(S_d)\right)
+\frac12\E\left( g'(S_d)\left(\sum_{i=1}^{d-1} \eps_i S_i\right)^2\right) + O(\eps^3)
\\
&=\E(G(S_d))-\frac12\E\left(S_d g(S_d)\right)\sum_{i=1}^{d-1}\eps_i^2
+\frac12 \E( g'(S_d))\sum_{i=1}^{d-1} \eps_i^2 + O(\eps^3)
\end{split}
\end{align}
where we used that $\E(S_i)=0$, $\E(S_i^2)=1$, and the fact that $S_i$ and $S_j$ are independent for $i\neq j$ so that $\E(S_ig(S_d))=\E(S_i)\E(g(S_d))=0$.
In particular, we obtain
\begin{align}
\E(G(Xw))&=\E(G(S_d))+\frac12 \left(\sum_i \eps_i^2\right) (\E(g'(S_d)-\E\left(S_dg(S_d)\right) + O(\eps^3).
\end{align}
We conclude that under the condition
\begin{align}
\E(g'(S_d))-\E\left(S_dg(S_d)\right)\neq 0
\end{align}
the function $w\to \E(G(wX))$  has a local extremum at $w=w_0$
and it is strictly  convex or concave around $w_0$.

\paragraph{The key lemma.}
We now generalize the reasoning above to the perturbed setting.
To simplify this further, we first assume $A=\mathrm{Id}$ so that 
\begin{align}\label{eq:data_id}
X = S + \eta h(S).
\end{align}
We also remove the linear whitening operation involved in $H$ and instead consider 
\begin{align}\label{eq:defHetah}
\tilde{H}_{\eta, h} ( w)=\E G\left(\frac{wX}{\sigma_w}\right)
\end{align}
where we used the shorthand $\sigma_w = \sqrt{\E((wX)^2)}$.
Clearly the function $\tilde{H}$ is invariant under rescaling of the argument, i.e., homogeneous of degree 0. 
Instead of restricting it to the sphere, we define the function $\bar{H}:\R^{d-1}\to \R$ given
by
\begin{align}\label{eq:defbarH}
 \bar{H}(\eps) = \tilde{H}_{\eta, h} ((\eps, 1)^\top)
\end{align} 
around $\eps=0$. Then our goal is to show that $\bar{H}$ has an extremum close to $\eps=0$ which allows us to  approximately recover
the independent component $S_d$.
We now prove the following Lemma.
\begin{lemma}\label{le:linearized}
Assume the contrast function $G$
and the distribution of the sources $S$ satisfy the Assumptions~\ref{as:ICA1}, \ref{as:ICA3}, and \ref{as:ICA4} and assume that
$\lVert h\rVert_{\P,q}\leq 1$ for $q=\max(3,d_g)$.
Define $v\in \R^{d-1}$ by 
\begin{align}
v_i = \E(h_d(S)g'(S_d) S_i + g(S_d) h_i(S)).
\end{align}
Then, there is $\eta_0>0$ (depending on all problem constants, e.g., $d$, $\alpha$, $q$, $M$ but not on $h$) such that for $\eta\leq \eta_0$ 
the function $\bar{H}$ has a local extremum $\eps_{0}$
in the vicinity of $\eps=0$ such that $|\eps_0|=O(\eta)$.
This extremum satisfies
\begin{align}\label{eq:local_extremum}
\eps_{0} = \frac{\eta v}{\alpha} + O(\eta^2).
\end{align}
%In addition there is for $0<\kappa<\kappa_0$ (depending on all problem constants except $\eta$) along with constanct $\Xi$, $\mu$, and $\eta_0$ such that for all $\eta\leq \min(\eta_0, c\kappa)$
%\begin{align}\label{eq:bounds_B_kappa}
%\max_{\eps \in B_{\kappa}(0)} \sign(\alpha)\bar{H}(\eps)\leq 
%\sign(\alpha)\bar{H}(0)+\Xi\eta^2
%\\
%\label{eq:bounds_B_kappa2}
%\max_{\eps \in \partial B_{\kappa}(0)} \sign(\alpha)\bar{H}(\eps)\leq \sign(\alpha)\bar{H}(0)-\frac{|\alpha|\kappa^2}{16}+\Xi\eta^2.
%\end{align}
%Moreover, the bound
%\begin{align}
%\nabla \bar{H}(\eps) = O(\eta+|\eps|)
%\end{align}
%holds.
Moreover there
is $\eta_0'>0$ (depending on the same quantities as $\eta_0$) such that for $\eta\leq \eta_0'$ 
there is a radius 
$\kappa>0$  such that $\bar{H}$ is strictly convex or concave on 
$B_{\kappa}(0)$ and satisfies
\begin{align}\begin{split} \label{eq:convexity_linaer}
D^2\bar{H}(\eps)\geq \frac{|\alpha_d|}{2} \cdot\mathrm{Id} \quad \text{for  $\alpha_d<0$ }\\
D^2\bar{H}(\eps)\leq \frac{-\alpha_d}{2} \cdot\mathrm{Id} \quad \text{for  $\alpha_d>0$ }.
\end{split}
\end{align}
\end{lemma}

\begin{proof}
%We want to consider $\E(G(wX))$ in the vicinity of 
%$w=e_d$ so we set $w=(\eps_1, \eps_2, \ldots, \eps_{d-1},1)^\top$
%we are mostly interested in $\eps_i$ small and we will 
%always assume $|\eps_i|\leq 1$.
Since we are interested in small $\eps$ we will always assume that $|\eps|\leq 1$.
We will  denote $w=(\eps,1)^\top$.
Then we have
\begin{align}
wX = S_d +\sum_{i=1}^{d-1} \eps_i S_i + \eta w\cdot  h(S).
\end{align}
%As standardization is non-trivial in this case due to the perturbation $h$
%we do not whiten the data and do not normalize $w$ and instead consider the function
%$w\to \E(G(wX/ \sigma_{w}))$
%where we used the shorthand $\sigma_w = \sqrt{\E((wX)^2)}$.
We define $\Omega = \E(h(S)^\top h(S))$, i.e., the covariance of the nonlinear part.
We find
\begin{align}
E((wX)^2)=
\E\left(S_d+\sum_{i=1}^{d-1} \eps_i S_i+\eta w\cdot h(S)\right)^2
= 1+ \sum_{i=1}^{d-1} \eps_i^2+\eta^2 w\cdot \Omega w
\end{align}
where we used that $S_i$ are unit variance and uncorrelated and
$\E(S^\top h(S))=0$.
In particular, we conclude that
\begin{align}
\sigma_{w}=\sqrt{1+ \sum_{i=1}^{d-1} \eps_i^2+\eta^2 w\cdot \Omega w}.
\end{align}
Note that $|\Omega| \leq \lVert h\rVert_{\P,2}^2\leq \lVert h\rVert_{\P,q}^2\leq 1$ by assumption.
Let us introduce some notation.
We consider the shorthand
\begin{align}
S_\eps=\sum_{i=1}^{d-1} \eps_i S_i.
\end{align}
We also use
\begin{align}
h_w(S)=w\cdot h(S), \quad h_\eps(S)=\sum_{i=1}^{d-1} \eps_i h_i(S).
\end{align}
Since $w=(\eps_1,\ldots, \eps_{d-1},1)^\top$ we have
\begin{align}
h_w(S)=h_\eps(S)+h_d(S).
\end{align}
Now we  perform a second order Taylor expansion of $G$ 
around $S_d$ with remainder term.
We obtain that for some  $\xi \in [S_d, wX/\sigma_w]$
\begin{align}
\begin{split}\label{eq:Gexp}
G\left(\frac{wX}{\sigma_w}\right)&=G\left(S_d\right)
+ g\left(S_d\right)
\frac{S_\eps + \eta h_w(S)+S_d (1-{\sigma_{w}})}{\sigma_{w}}
\\
&\quad + \frac12 g'\left({S_d}\right)
\left(\frac{S_\eps + \eta h_w(S)+ S_d(1-\sigma_w)}{\sigma_{w}}\right)^2
+
 \frac16 g''(\xi)
\left(\frac{S_\eps + \eta h_w(S)+S_d(1-\sigma_w)}{\sigma_{w}}\right)^3.
\end{split}
\end{align}
%We obtain for some  $\xi \in [S_1/\sqrt{1+\Omega_{11}}, wX]$
%\begin{align}
%\begin{split}
%G(wX)&=G\left(\frac{S_1}{\sqrt{1+\Omega_{11}}}\right)
%+ g\left(\frac{S_1}{\sqrt{1+\Omega_{11}}}\right)
%\frac{S_\eps + \eta h_w(S)+S_1 - S_1\frac{\sigma_{w}}{\sqrt{1+\Omega_{11}}}}{\sigma_{w}}
%\\
%&\quad + \frac12 g'\left(\frac{S_1}{\sqrt{1+\Omega_{11}}}\right)
%\left(\frac{S_\eps + \eta h_w(S)+S_1 - S_1\frac{\sigma_{w}}{\sqrt{1+\Omega_{11}}}}{\sigma_{w}}\right)^2
%+
% \frac16 g''(\xi)
%\left(\frac{S_\eps + \eta h_w(S)+S_1 - S_1\frac{\sigma_{w}}{\sqrt{1+\Omega_{11}}}}{\sigma_{w}}\right)^3.
%\end{split}
%\end{align}
Our goal is to extract the quadratic terms in $\eps$ and $\eta$ of this expression.
We now start to bound the error term and show that it is of order 3.
Since $\xi \in [S_d, wX/\sigma_w]$ and $\sigma_w>1$ we conclude that
$|\xi|\leq \max(|wX|,|S_d|)$.
Then we can control using Assumption~\ref{as:ICA1}
\begin{align}
\begin{split}
|g''(\xi)|\leq C_g (1 + |\xi|^{\max(d_g-3,0)})
&\leq C_g(1 + \max(|wX|, |S_d|)^{\max(d_g-3,0)})
\\
&\leq C(1+|S|^{\max(d_g-3,0)}+|\eta h(S)|^{\max(d_g-3,0)}).
\end{split}
\end{align}
We have  the simple bound
\begin{align}
\left| \sigma_w - 1 \right|\leq
|\eps|^2+|\eta|^2 |w|^2\cdot | \Omega|
\leq |\eps|^2+|\eta|^2 |w|^2. 
\end{align}
Then we can bound (using $\eta^2|\eps|+|\eps|^2\eta\leq |\eps|^3+\eta^3$)
\begin{align}
\left|\frac{S_\eps + \eta h_w(S)+S_d(1-\sigma_w)}{\sigma_{w}}\right|^3
\leq C(|\eps|^3 + \eta^3)|S|^{3}+\eta^3|h(S)|^{3}
\end{align}
This implies (recall that $q=\max(3, d_g)$)
\begin{align}\label{eq:err1}
\frac16 \left|g''(\xi)
\left(\frac{S_\eps + \eta h_w(S)+S_d(1-\sigma_w)}{\sigma_{w}}\right)^3\right|
\leq C(d,C_g, q) (|\eps|^3 + \eta^3)\left(1 + |S|^{q}+|h(S)|^{q}\right).
\end{align}
Next, we consider the second term 
\begin{align}
\frac12 g'\left({S_d}\right)
\left(\frac{S_\eps + \eta h_w(S)+ S_d(1-\sigma_w)}{\sigma_{w}}\right)^2
\end{align}
which we approximate to 2nd order in 
$\eps$ and $\eta$ and put all terms of order 3 and higher in the error term. Note that $S_\eps=O(\eps)$. We expand $\eta h_w(S)
=\eta h_d(S)+\eta h_\eps(S)$ where
$\eta h_d(S)=O(\eta)$ and  $\eta h_\eps(S)=O(\eta\eps)$. Finally $(1-\sigma_w)=O(|\eps|^2+\eta^2)$.
We conclude that 
\begin{align}
\left|\left(\frac{S_\eps + \eta h_w(S)+ S_d(1-\sigma_w)}{\sigma_{w}}\right)^2
- \left(\frac{S_\eps + \eta h_d}{\sigma_w}\right)^2\right|\leq C(|S|^2+|h(S)|^2)
\end{align}
where we again bounded mixed terms by, e.g.,  $|\eps|^2\eta \leq |\eps|^3+\eta^3$.
Using in addition that $1-\sigma_w^{-1}\leq C( |\eps|^2 + \eta^2)$ to replace $\sigma_w$ in the denominator by $1$ 
up to third order error terms,
we find
\begin{align}\label{eq:err2}
\left|g'(S_d)\left(\frac{S_\eps + \eta h_w(S)+S_d(1-\sigma_w)}{\sigma_w}\right)^2-
g'(S_d)\left(S_\eps + \eta h_d(S)\right)^2
\right|
\leq C (|\eps|^3 + \eta^3) (1+|S|^{\max(d_g-2,0)}) (|S|^2+|h(S)|^2).
\end{align}
Finally we consider the term proportional to $g(S_d)$. 
Here we need the sharper bound 
\begin{align}
\left| \sigma_w - 1 - \frac{|\eps|^2}{2}-\frac{\eta^2}{2}w\cdot \Omega w \right|\leq
2|\eps|^4+2\eta^4 (|w|^2 | \Omega|)^2\leq 2|\eps|^4+2\eta^4|w|^4
\end{align}
which implies using $w=e_d+O(\eps)$ that $w\cdot \Omega w= \Omega_{dd}+O(\eps)$ and moreover (recall $|\Omega|\leq 1$)
\begin{align}
\left| \sigma_w - 1 - \frac{|\eps|^2}{2}-\frac{\eta^2}{2}\Omega_{dd} \right|\leq
2|\eps|^4+2\eta^4 |w|^4+3 |\eps|\eta^2 .
\end{align}
Then we get, again keeping terms up to order 2 in $\eps$ or $\eta$
\begin{align}\label{eq:err3}
\begin{split}
&\left|g(S_d) \frac{S_\eps + \eta h_w(S)+S_d (1-{\sigma_{w}})}{\sigma_{w}}
- g(S_d)\left(S_\eps + \eta h_w(S)-\frac12 (|\eps|^2+\eta^2 \Omega_{dd})S_d\right)\right|
\\
&\qquad \qquad\leq C (|\eps|^3+\eta^3)|S|^{\max(d_g-1,0)}(|S|+|h(S)|).
\end{split}
\end{align}
Using the bounds \eqref{eq:err1}, \eqref{eq:err2}, and \eqref{eq:err3}
in \eqref{eq:Gexp}
we obtain
\begin{align}
\begin{split}
\left|
G(wX)-G(S_d)- g(S_d)\left(S_\eps + \eta h_w(S)-\frac12(|\eps|^2+\eta^2 \Omega_{dd}) S_d\right)-\frac12
g'(S_d)\left(S_\eps + \eta h_d(S)\right)^2
\right|
\\
\leq
C(|\eps|^3+|\eta|^3) (1+|S|^{q}+ |h(S)|^{q}).
\end{split}
\end{align}
Let us call the term between the absolute value on the left-hand side
$T$.
Then we can bound using Assumption~\ref{as:ICA3}
\begin{align}\label{eq:taylor_final}
|\E T|\leq \E|T| \leq C(|\eps|^3+\eta^3)\E(|S|^{q}+ |h(S)|^{q})
\leq C (|\eps|^3+|\eta|^3) (M+\lVert h\rVert_{\P,q}^q)
= \Xi(|\eps|^3+|\eta|^3)
\end{align}
where $\Xi=C(M+2)$ was introduced for future reference and depends on $C_g$, $d$,
$q$, and the moment bound  $M$  but is independent of $\eps$ and $\eta$ and $h$.
We observe next, using  $\E(S_dS_\eps)=0$, $\E(g(S_dS_\eps)=0$ that
\begin{align}
\begin{split}\label{eq:boundET}
\E T&=\E G\left(\frac{wX}{\sigma_w}\right) -\E(G(S_d))-\eta  \E(h_w(S)g(S_d))
+\frac12(|\eps|^2+\eta^2 \Omega_{dd})\E(S_dg(S_d))
\\
&\quad
-\frac12|\eps|^2\E (g'(S_d)) -\eta \E(h_d(S) S_\eps g'(S_d))
-\frac12\eta^2 \E(g'(S_d)h_d(S)^2)
\\
&=\E G\left(\frac{wX}{\sigma_w}\right)  -\E(G(S_d))-\eta  \E(h_d(S)g(S_d))
+\frac12(\eta^2 \Omega_{dd})\E(S_dg(S_d))
-\frac12\eta^2 \E(g'(S_d)h_d(S)^2)
\\
&\quad
-\eta \E(h_d(S) S_\eps g'(S_d))-
\eta  \E(h_\eps(S)g(S_d))
+\frac12|\eps|^2\left(\E g(S_d)S_1-\E g'(S_d)\right).
\end{split}
\end{align}
We thus obtained an expansion of $H(w)=\E G\left({wX}/{\sigma_w}\right)$
up to second order in $\eps$. Note that by plugging in $\eps_i=0$ for all $i$, i.e., 
$w=e_d$, in \eqref{eq:taylor_final}
\begin{align}\label{eq:boundETd}
\left|\E G\left(\frac{e_dX}{\sigma_{e_d}}\right)  -\E(G(S_d))-\eta  \E(h_d(S)g(S_d))
+\frac12(\eta^2 \Omega_{dd})\E(S_dg(S_d))
-\frac12\eta^2 \E(g'(S_d)h_d(S)^2)\right|\leq \Xi \eta^3 
\end{align}
so we conclude using the triangle inequality from \eqref{eq:taylor_final},  \eqref{eq:boundET}, and 
\eqref{eq:boundETd} that 
\begin{align}\label{eq:final_H}
\left|\E G\left(\frac{wX}{\sigma_w}\right) 
- \E G\left(\frac{e_dX}{\sigma_{e_d}}\right) 
-\eta \E(h_d(S) S_\eps g'(S_d))
- \eta \E(h_\eps(S) g(S_d))
+\frac12|\eps|^2\left(\E g(S_d)S_d-\E g'(S_d)\right)
\right|\leq \Xi(|\eps|^3+2\eta^3).
\end{align}
As a next step, we consider the derivatives of $H$.
Since the reasoning is similar to the steps above, we provide slightly fewer details, i.e., we hide integrable
terms in the $O$ notation. 
First we observe that 
\begin{align}\label{eq:deriv_sigma}
\partial_i \sigma_w = \frac{\eps_i+\eta^2 (\Omega w)_i}{\sigma_w}. 
\end{align}
From here we obtain
\begin{align}\label{eq:deriv_H}
\partial_i G\left(\frac{wX}{\sigma_w}\right)=\left( \frac{S_i+ \eta h_i(S)}{\sigma_w}
-\frac{(\eps_i+\eta^2 (\Omega w)_i)(wS + \eta h_w(S))}{\sigma_w^3}\right)
 g \left(\frac{wX}{\sigma_w}\right).
\end{align}
We apply a first order Taylor expansion to $g$ and obtain the bound
$ g \left(\frac{wX}{\sigma_w}\right)=g(S_d)+O(|\eps|+\eta)$.
Thus, we find
\begin{align}\label{eq:bound_DH}
\partial_i \bar{H}(\eps)=
\partial_i 
\E G\left(\frac{wX}{\sigma_w}\right)=\frac{1}{\sigma_w}\E(S_ig(S_d))+
O(|\eps|+\eta)=O(|\eps|+\eta)
\end{align}
where we used that $S_j$ and $S_d$ are independent.

Next we show  using   the bound \eqref{eq:boundg} from Assumption~\ref{as:ICA1}  that the function $\eps\to w=(\eps, 1)^\top\to \E(G(wX/\sigma_w))$ is strictly  convex or concave  around
$\eps=0$ for $\eta$ sufficiently small. 
For this, we need to find an expression for the second derivatives of $G$.
To keep the length of the formulas manageable, we hide the $\eta^2 (\Omega w)_i$ as a $O(\eta^2)$ term.
We find for $1\leq i,j\leq d-1$ using \eqref{eq:deriv_H}
and \eqref{eq:deriv_sigma}
\begin{align}
\begin{split}
\partial_i\partial_j 
&G\left(\frac{wX}{\sigma_w}\right)
=\partial_i \left(\left( \frac{S_j+ \eta h_j(S)}{\sigma_W}
-\frac{(\eps_j+\eta^2(\Omega w)_j)(wS + \eta h_w(S))}{\sigma_w^3}\right)
 g \left(\frac{wX}{\sigma_w}\right)\right)
\\
& =
 \left( \frac{S_j+\eta h_j(S)}{\sigma_w}+ 
-\frac{\eps_j(wS + \eta h_w(S))}{\sigma_w^3}\right)
\left( \frac{S_i+ \eta h_i(S)}{\sigma_w}
-\frac{\eps_i(wS + \eta h_w(S))}{\sigma_w^3}\right)
 g' \left(\frac{wX}{\sigma_w}\right)
 \\
&  -
 \left( \frac{\eps_j (S_i+\eta h_i(S) )
 +\delta_{ij}(wS+\eta h_w(S))
 +\eps_i( S_j+\eta h_j(S)) }{\sigma_w^3}-\frac{3\eps_i\eps_j(wS+\eta h_w(S))}{\sigma_w^5}\right) g \left(\frac{wX}{\sigma_w}\right)+O(\eta^2)
 \\
 &=
 S_j S_i g' \left(\frac{wX}{\sigma_w}\right)
 -\delta_{ij}S_1
 g \left(\frac{wX}{\sigma_w}\right)
 +O(|\eps|+\eta).
 \end{split}
\end{align}
Here, we used again that $\sigma_w = 1+O(|\eps|+\eta)$.
Now we apply a Taylor expansion to $g$ and $g'$
and obtain $g(wX/\sigma_w)=g(S_d)+O(|\eps|+\eta)$ and
similarly for $g'$ (here similar power counting as in the first part implies that the highest moment that needs to be bounded is $q=\max(3, d_g)$, we do not show this in full detail here).
Then we obtain
\begin{align}
\begin{split}
\partial_i\partial_j \bar{H}(\eps)=
\partial_i \partial_j \E
G\left(\frac{wX}{\sigma_w}\right)
&= \E\left(S_iS_j g'\left(\frac{wX}{\sigma_w}\right)
\right)
- \delta_{ij}
\E\left(S_d g\left(S_d\right)
\right)+O(|\eps|+\eta)
\\
&=\delta_{ij} \E(g'(S_d)-S_dg(S_d))+ O(|\eps|+\eta)
= -\alpha_d \delta_{ij}+O(|\eps|+\eta).
\end{split}
\end{align}
In particular, we find
\begin{align}\label{eq:bound_D2H}
| D^2 \bar{H}(\eps)  + \alpha_d\cdot  \mathrm{Id}| =O(|\eps|+\eta).
\end{align}
We conclude that for $\eta<\eta_0$ with $\eta_0$ sufficiently small, there is for some $\kappa>0$ a neighborhood
$B_\kappa(0)$ such that the function 
$B_\kappa(0)\ni \eps \to \bar{H}(\eps)$
is strictly convex or concave  (depending on the sign of
$\alpha_d=\E(S_dg(S_d)-g'(S_d))$) with $D^2\bar{H}\geq |\alpha_d|/2 \cdot \id$ in the convex case. We emphasize that $\kappa$ is independent of $\eta$ as soon as $\eta\leq \eta_0$ is sufficiently small.

It remains to prove the existence of a maximum or minimum and the expansion \eqref{eq:local_extremum}.
To achieve this, we compare $\bar{H}$ with its expansion to second order, i.e., we define 
\begin{align}
f(\eps) = \eta \E(h_d(S) S_\eps g'(S_d))+\eta \E(h_\eps(S)g(S_d))
-\frac12|\eps|^2\left(\E g(S_d)S_d-\E g'(S_d)\right)
= \eps \eta v -\frac12 \alpha|\eps|^2 
\end{align}
where we recall that we defined  $v\in \R^{d-1}$ 
by $v_i=\E(h_d(S) S_i g'(S_d))+\E(h_i(S)g(S_d))$ for $i=1,\ldots, d-1$ 
and $\alpha_d= \E g(S_d)S_d-\E g'(S_d)$.
Then we can rewrite \eqref{eq:final_H} as
\begin{align}\label{eq:boundfH}
\left|\bar{H}(\eps)-\bar{H}(0)-f(\eps)\right|< \Xi(|\eps|^3+2\eta^3).
\end{align}
In other words $\bar{H}$ agrees with $f$ up to a constant term and error terms of order 3 in $\eps$ and $\eta$.
If we just use this bound on $|\bar{H}-f|$ 
we could prove \eqref{eq:local_extremum} but only with an error term of order $\eta^{\frac32}$. To obtain the better rate $\eta^2$ we need to  consider a second order expansion.

We assume now that $\alpha_d>0$ such that $\bar{H}$ and $f$ are concave
in a neighborhood $B_{\kappa}(0)$ (the proof for $\alpha_d<0$ is very similar) for $\eta<\eta_0$. For $\alpha_d<0$
a similar reasoning applies and only some inequalities are reversed.

We first expand the function $f$. The relation $\alpha_d> 0$ implies that $f$ is maximized  at 
\begin{align}\label{eq:fmax}
\eps_{\max}=\frac{\eta v}{\alpha_d}, \quad \text{and} 
\quad 
f(\eps)\leq f(\eps_{\max})=\frac{\eta^2|v|^2}{2\alpha_d}.
\end{align} 
and we have for
$\eps = \eps_{\max}+\Delta\eps$ the expansion
\begin{align}\label{eq:exp_f}
f(\eps)=f(\eps_{\max}+\Delta \eps)
=f(\eps_{\max}) -\frac12 \alpha_d |\Delta\eps|^2.
\end{align}
%Next we show that $\bar{H}(\eps)$ has a local extremum close
%to $\eps_{\max}$ for $\eta$ sufficiently small.

Now we consider similar expansions for $\bar{H}$.
For concreteness, we introduce the constant $\Xi_2$
such that 
\begin{align}
  % | D \bar{H}(\eps)|&\leq \Xi_1 (|\eps|+\eta)
%\\
| D^2 \bar{H}(\eps) +  \alpha_d\cdot  \mathrm{Id}| &\leq \Xi_2(|\eps|+\eta).
\end{align}
Such a constant exists by 
%\eqref{eq:bound_DH} and 
\eqref{eq:bound_D2H}.

Recall that we assumed  that $\alpha_d>0$ such that $\bar{H}$ is concave
in a neighborhood $B_{\kappa}(0)$ for $\eta<\eta_0$. 
Suppose $\eps_0$ is the unique global maximum of $\bar{H}$ on $\overline{B_\rho(0)}$
for  
\begin{align}
    \label{eq:def_rho}
    \rho= \frac{2\eta(|v|+1)}{|\alpha_d|}
\end{align}
where we assume that $\eta$ is sufficiently small such that
$\bar{H}$ is uniformly convex on $B_\rho$.
Note that then either $D\bar{H}(\eps_0)=0$ or $\eps_0\in \partial B_\rho(0)$ and $D\bar{H}(\eps_0)(\eps-\eps_0)<0$ for all $\eps \in B_\rho(0)$. 
Now, the general heuristic is that $\bar{H}$ roughly behaves like a parabola with vertex $\eps_0$  and $f$ is a parabola with vertex $\eps_{\max}$ and both parabolas have 
approximately the same second derivative. Then, the distance 
between the two parabolas will increase for large arguments, leading to a contradiction to \eqref{eq:boundfH}.

We now define the point $\eps_1$ as the intersection of the ray from 
$\eps_{\max}$ to $\eps_0$ with the set $\partial B_{2\rho}(0)$. We define 
\begin{align}
    \delta = |\eps_{\max}-\eps_0|
\end{align} (the quantity we want to bound) and
\begin{align}
    \mu=|\eps_0-\eps_1|\geq \rho
\end{align}
(since $\eps_0\in B_\rho(0)$ and
$|\eps_1|=2\rho $). Note that $\eps_1$ is on the ray from $\eps_{\max}$
to $\eps_0$ so we have $|\eps_1-\eps_{\max}|=
|\eps_1-\eps_{0}|+|\eps_0-\eps_{\max}|=\delta+\mu$.
We then find using \eqref{eq:exp_f}
\begin{align}\label{eq:f_eps0_eps1}
    f(\eps_0)-f(\eps_1)=
    =(f(\eps_{\max})-\frac12 \alpha_d \delta^2)-(f(\eps_{\max})- \frac12 \alpha_d
    (\delta+\mu)^2)
    =-\frac12 \alpha_d \delta^2+ \frac12 \alpha_d
    (\delta+\mu)^2
    =
    \frac12 \alpha_d(\mu^2+2\mu \delta).
\end{align}

We now derive a similar bound for $\bar{H}$ for which we need the second order Taylor expansion with integral remainder which reads for $g:\R^d\to \R$ with $g\in C^2$ as follows
\begin{align}\label{eq:second_order}
    g(x)=g(x_0) + Dg(x_0)(x-x_0)+
    \int_0^1 (1-t) (x-x_0)^\top D^2g(x_0+t(x-x_0)) (x-x_0) \, \d t.
\end{align}
We apply this with $\eps_0$ and $\eps_1$. 
We observe that since $\eps_1-\eps_0$ and 
$\eps_{\max}-\eps_0$ point in opposite directions, the relation
$D\bar{H}(\eps_0)(\eps_{\max}-\eps_0)\leq 0$ (since $\eps_{\max}\in B_\rho(0)$)
implies 
\begin{align}
    D\bar{H}(\eps_0)(\eps_1-\eps_0)\geq 0. 
\end{align}
Then we find from the second order expansion \eqref{eq:second_order} using the last display and \eqref{eq:bound_D2H}
\begin{align}
\begin{split}\label{eq:H_eps0_eps1}
   \bar{H}(\eps_1)-\bar{H}(\eps_0)
   &=  D\bar{H}(\eps_0)(\eps_1-\eps_0)+
    \int_0^1 (1-t) (\eps_1-\eps_0)^\top D^2\bar{H}(\eps_0+t(\eps_1-\eps_0)) (\eps_1-\eps_0) \, \d t
\\
&\geq 0
-\left(\alpha_d +\Xi_2(2\rho + \eta)\right)(\eps_1-\eps_0)^2 \int_{0}^1 (1-t)\,\d t
\\
&=-\frac12\left(\alpha_d +\Xi_2(2\rho + \eta)\right)\mu^2.
   \end{split}
\end{align}
Using \eqref{eq:boundfH} followed by \eqref{eq:f_eps0_eps1} 
and \eqref{eq:H_eps0_eps1} we now find
\begin{align}
\begin{split}
    4\Xi ( 4\rho^3+\eta^3)
    &\geq |\bar{H}(\eps_0)-\bar{H}(0)-f(\eps_0)|
    + |\bar{H}(\eps_1)-\bar{H}(0)-f(\eps_1)|
   \\
   &\geq -\left(\bar{H}(\eps_0)-\bar{H}(0)-f(\eps_0)\right)
    +\left(\bar{H}(\eps_1)-\bar{H}(0)-f(\eps_1)\right)
    \\
    &=(\bar{H}(\eps_1)-\bar{H}(\eps_0))
    +(f(\eps_0)-f(\eps_1))
    \\
   & \geq 
    -\frac12\left(\alpha_d +\Xi_2(2\rho + \eta)\right)\mu^2
    +\frac12 \alpha_d(\mu^2+2\mu \delta)
    \\
   & = \alpha_d \mu\delta -\frac12 \Xi_2(2\rho + \eta)\mu^2.
    \end{split}
\end{align}
We conclude using $\rho\leq \mu\leq 2\rho$ and the definition
\eqref{eq:def_rho} that
\begin{align}
\begin{split}
    \delta &\leq \frac{  4\Xi ( 4\rho^3+\eta^3)}{\alpha_d \mu}
    +\frac1{2\alpha_d} \Xi_2(2\rho + \eta)\mu
    \\
    &\leq \frac{4\Xi}{\alpha_d}\left(4\rho^2 +\frac{\eta^3}{\rho}\right)
    +\frac1{\alpha_d} \Xi_2(2\rho^2+\eta\rho)
    \\
    &\leq \frac{4\Xi}{\alpha_d}\left(\frac{16\eta^2(|v|+1)^2}{\alpha_d^2} +\frac{\alpha_d\eta^2}{2(|v|+1)}\right)
    +\frac{ \Xi_2}{\alpha_d}\left(\frac{8\eta^2(|v|+1)^2}{\alpha_d^2}+
    \frac{2\eta^2(|v|+1)}{\alpha_d}\right)
    \\
    &\leq C\eta^2.
    \end{split}
\end{align}
For $\eta$ sufficiently small this implies that $\delta<\rho/2$ and thus $\eps_0\in B_\rho(0)$ (i.e., in the interior and not on the boundary) so that $\eps_0$ really is a local maximum of $\bar{H}$.
The proof for $\alpha_d<0$ follows similarly.

\end{proof}

We now extend Lemma~\ref{le:linearized} by considering  $A\neq \mathrm{Id}$, including the whitening, and considering the function on the sphere, thus proving Theorem~\ref{th:pert}.
%For convenience of the reader we restate Theorem~\ref{th:pert} here.
%%TODOTODO
%%and add one additional bound that is needed in the proof of Theorem~\ref{th:deflated_ica} below.
%%TODO has to be added again!!
%\begin{theorem}\label{th:ICA_full}
%Let $w_0\propto \Sigma_X^{\frac12} A^{-\top} e_d$ such that $|w_0|=1$.
%Assume that $\alpha_d =\E(g(S_d)S_d-g'(S_d))\neq 0$ and
%define $\lambda=\mathrm{sign}(\alpha_d)$.
%Then, there is $\kappa_0$ such that  for $0\leq \kappa\leq \kappa_0$
%there is  neighbourhood $U_\kappa\subset S^{d-1}$ of $w_0$ 
%such that $B_{c_1\kappa}(w_0)\subset U_\kappa$
%for some sufficiently small $c_1$ (independent of $\kappa$)
%and constants $\Xi>0$, $\eta_0>0$ and $c_2>0$ (all not depending on $\kappa$)
%such that for $\eta<\min(\eta_0,c_2\kappa)$ the following bounds hold
%\begin{align}\label{eq:U_bound_trafo1}
%\max_{w\in U} \lambda H_{\eta}(w)\leq \lambda H_\eta(w_0)+\Xi\eta^2
%\\
%\label{eq:U_bound_trafo2}
%\max_{w\in U} \lambda H_{\eta}(w)\leq \lambda H_\eta(w_0)-\frac{|\alpha|\kappa}{16}+\Xi\eta^2.
%\end{align}
%\end{theorem}
% TODO end of material to add

\begin{proof}[Proof of Theorem~\ref{th:pert}]
Recall that we defined $H:S^{d-1}\to\R$ by
\begin{align}
H_\eta(w)=\E G(w^\top \Sigma_X^{-\frac12} X)
\end{align}
where $\Sigma_X=\E(XX^\top)=AA^\top + \eta^2\Omega$
with $\Omega=\E(h(S)h(S)^\top)$. Note that $\Sigma_X$ depends implicitly on $\eta$ so we will indicate this dependence in the following for some quantities.
We now relate this to the setting in Lemma~\ref{le:linearized}.
The function $H$ defines a map  on the manifold $S^{d-1}$. We analyze its properties by considering a suitable chart. 
We define the map $T:\R^{d-1}\to S^{d-1}$  by
\begin{align}
T_\eta(\eps)=\frac{\Sigma_X^{\frac12}A^{-\top}\begin{pmatrix}
\eps \\ 1\end{pmatrix}}
{\left|\Sigma_X^{\frac12}A^{-\top}\begin{pmatrix}
\eps \\ 1\end{pmatrix}\right|}.
\end{align}
This map defines a chart locally around $\eps=0$. 
Recall that we defined in \eqref{eq:defbarH} (indicating the parameter dependence for clarity)
\begin{align}
\bar{H}_{\eta,h}(\eps)
=
\tilde{H}_{\eta, h} \left(
%\begin{pmatrix}\eps \\ 1 \end{pmatrix}
w\right)
=\E G\left(
\frac{
%\begin{pmatrix}\eps \\ 1 \end{pmatrix}^\top 
w^\top(S+\eta h(S))}{\sigma_w}\right)
\end{align}
where $w=(\eps, 1)^\top$  and $\sigma_w = \E (wX)^2$ with  $X=S+\eta h(S)$.
Then the relation
\begin{align}\label{eq:chart_map}
H_\eta(T_\eta(\eps))=\bar{H}_{\eta,A^{-1}h}(\eps)
\end{align}
holds. 
Indeed, writing $w=(\eps, 1)^\top$ we obtain

\begin{align}
\begin{split}
H_\eta(T_\eta(\eps))
&=
\E G\left(\frac{\left(\Sigma_X^{\frac12}A^{-\top} w \right)^\top \Sigma_X^{-\frac12}X}{ \left|\Sigma_X^{\frac12}A^{-\top}w\right|}
\right)
\\
&=
\E G\left(\frac{w^\top A^{-1} (AS+\eta h(S))}{ \left|\Sigma_\eta^{\frac12}A^{-\top}w\right|}
\right)
\\
&=
\E G\left(\frac{w^\top (S+\eta A^{-1}h(S))}{ \sigma_w}
\right)=\bar{H}_{\eta,A^{-1}h}(\eps)
\end{split}
\end{align}
where we used that (recall $\E(Sh(S)^\top)=0$)
\begin{align}
\begin{split}
\sigma_w^2=\E((w^\top (S+\eta A^{-1}h(S)))^2)
&=w\cdot \E\left((S+\eta A^{-1}h(S))(S^\top+\eta h(S)^\top A^{-\top})\right) w
\\
&=w\cdot (\mathrm{Id}+\eta^2 A^{-1}\Omega A^{-\top}w
=
(A^{-\top}w)\cdot  (AA^\top+\eta^2 \Omega)A^{-\top}w\
\\
&=(A^{-\top}w)\cdot\Sigma_X A^{-\top}w=\left|\Sigma_X^{\frac12}A^{-\top}w\right|^2.
\end{split}
\end{align}
Note that this is not surprising as both terms were  chosen such that the argument of $G$ has unit variance.
Note that $|A^{-1}h(S)|\leq |A^{-1}|\cdot |h(S)|$ so we find $\lVert A^{-1}h\rVert_{\P,q}^q\leq |A^{-1}| \cdot \lVert h\rVert_{\P,q}$. 
Now we apply Lemma~\ref{le:linearized} to the function 
$\bar{H}_{\eta,A^{-1}h}(\eps)=\bar{H}_{|A^{-1}|\eta, |A^{-1}|^{-1}A^{-1}h}(\eps)$
where $\lVert |A^{-1}|^{-1}A^{-1}h\rVert_{\P,q}\leq 1$. 
Lemma~\ref{le:linearized} implies that 
$\bar{H}$ has a local extremum at some $\eps_0=\eta\alpha_d^{-1}v+O(\eta^\frac32)$ for $\eta<\eta_0$ and some $\eta_0$
where $\alpha_d=\E g(S_d)S_d-\E g'(S_d)$
and 
\begin{align}
v_i = \E((A^{-1}h)_d(S)g'(S_d)S_i+
g(S_d)(A^{-1}h)_i(S)).
\end{align}
Thus 
$H_\eta $ has a local extremum at 
\begin{align}
w_d=T_\eta(\eps) \propto \Sigma_X^{\frac12}A^{-\top}\begin{pmatrix}
\eta\alpha_d^{-1}v \\ 1 
\end{pmatrix}+O(\eta^{2}).
\end{align}
Since $\Sigma_X=AA^\top +O(\eta^2)$ we also find (recall $\bar{w_d}=(AA^\top)^{\frac12}A^{-\top}e_d$ by \eqref{eq:def_barw})
that $|w_d-\bar{w_d}|=O(\eta)$.
The estimated independent component is  given by
\begin{align}
w_d^\top \Sigma^{-\frac12}_X X=\begin{pmatrix}
\eta\alpha_d^{-1}v \\ 1
\end{pmatrix}\cdot S + \eta \begin{pmatrix}
\eta\alpha_d^{-1}v \\ 1
\end{pmatrix}\cdot A^{-1}h(S)=S_d+O(\eta)
\end{align}
and 
\begin{align}
\Sigma_X^{\frac12} w_0 =A^{-\top}\begin{pmatrix}
\eta\alpha_d^{-1}v \\ 1 \end{pmatrix}+O(\eta^{2})=
 (A^{-1})_{d,:}+O(\eta),
\end{align}
i.e., we recover the $d$-th row of the unmixing matrix up to errors of order $\eta$.

%Equations \eqref{eq:U_bound_trafo1} and \eqref{eq:U_bound_trafo2}
%follow from \eqref{eq:bounds_B_kappa} and \eqref{eq:bounds_B_kappa2}
%by setting $U_\kappa = T_\eta(B_\kappa(0))$. The relation $B_{c_1\kappa}\subset U_\kappa$ follows from the bounded derivatives of $T_\eta^{-1}$.

Finally, we prove the convexity of $H$ around $\Sigma_X^{\frac12} A^{-\top}e_d$
for $\alpha<0$. Intuitively, this is not surprising as we proved the convexity of $\bar{H}$ around $\eps=0$ and the relation \eqref{eq:chart_map}
should lift this to the map $H$. The formal proof requires tools from differential geometry, which we use freely.
The proof relies on the following relation from 
Riemannian geometry
\begin{align}\label{eq:exp_hessian}
\mathrm{Hess}(H)=\sum_{i,j}
\frac{\partial^2 \bar{H}}{\partial \eps_i \,\partial\eps_j}\d \eps_i\otimes \d \eps_j
-\sum_{i,j,k}\Gamma^k_{ij}\frac{\partial \bar{H}}{\partial \eps_k}
\d \eps_i\otimes \d \eps_j.
\end{align}
Here $\Gamma^k_{ij}$ denotes the Christoffel symbols expressed in the chart $T_\eta$ where we use the induced metric on $S^{d-1}$ as a submanifold of $\R^d$.

We have shown in Lemma~\ref{le:linearized} that the matrix
with entries $D^2\bar{H}$ satisfies for $\alpha<0$
\begin{align}
D^2\bar{H}\geq |\alpha| \cdot \mathrm{Id}+O(\eta + |\eps|)
\end{align}
and that
\begin{align}
D \bar{H}=O(|\eps|+\eta).
\end{align}
It is straightforward to show using the calculations below that the Christoffel symbols are bounded (in fact small, but we do not need this). This and the last display imply that the last term in \eqref{eq:exp_hessian} is bounded $O(\eta+|\eps|)$.
We now consider the tangent vectors
$\partial_{\eps_i}$ which can be identified with the vector $\partial_{\eps_i}T_\eta(\eps)\in \R^d$ 
that is tangential to $S^{d-1}$.
We note that
\begin{align}
\left|\Sigma_X^{\frac12}A^{-\top}\begin{pmatrix}
\eps \\ 1\end{pmatrix}\right|^2
=
\begin{pmatrix}
\eps \\ 1\end{pmatrix}^\top
A^{-1}\Sigma_XA^{-\top}\begin{pmatrix}
\eps \\ 1\end{pmatrix}
=
\begin{pmatrix}
\eps \\ 1\end{pmatrix}^\top 
(\mathrm{Id}+\eta^2 A^{-1}\Omega A^{-\top})
\begin{pmatrix}
\eps \\ 1\end{pmatrix}
= 1 + |\eps|^2 +\eta^2 P(\eps)
\end{align}
where $P$ denotes a quadratic polynomial in $\eps$.
This implies
\begin{align}
\partial_{\eps_i} T_\eta(\eps)
=
\frac{\Sigma_X^{\frac12}A^{-\top}e_i}{\left|\Sigma_X^{\frac12}A^{-\top}\begin{pmatrix}
\eps \\ 1\end{pmatrix}\right|}
+O(|\eps|+\eta).
\end{align}
This implies
\begin{align}
g_{S^{d-1}}(\partial_{\eps_i},\partial_{\eps_j})
=\frac{e_i^\top A^{-1}\Sigma_XA^{-\top}e_j}{\left|\Sigma_X^{\frac12}A^{-\top}\begin{pmatrix}
\eps \\ 1\end{pmatrix}\right|^2} +O(|\eps|+\eta)
= \frac{\delta_{ij} +O(\eta^2)}{1+O(|\eps|^2+\eta^2}
+O(|\eps|+\eta)
=
\delta_{ij} +O(|\eps|+\eta).
\end{align}
This means that, up to higher-order terms, the metric $g$ agrees with the standard metric.
We conclude that for any tangent vector $Y=\sum_{i=1}^{d-1}y_i\partial_{\eps_i}$ we obtain
\begin{align}
\frac{\mathrm{Hess}(H)(Y,Y)}{g_{S^{d-1}}(Y,Y)}
\geq \frac{|\alpha|\cdot |y|^2}{|y|^2} +O(\eta+|\eps|).
\end{align}
We conclude that for $\eta\leq \eta_0$ 
there is a neighborhood of $T_\eta(0)\propto\Sigma_X^\frac12 A^{-\top} e_d$ where $\mathrm{Hess}(H)\geq |\alpha|/2 \cdot \mathrm{Id}$,
i.e. $H$ is strictly convex.
The proof for $\alpha>0 $ is similar.

\end{proof}
The proof of Theorem~\ref{th:pert_matrix} is now trivial.
\begin{proof}[Proof of Theorem~\ref{th:pert_matrix}]
By the assumptions of the Theorem and Theorem~\ref{th:pert}
we find for each $1\leq i\leq d$ a vector $w_i$ with $|w_i|=1$ such that $H$ has a local extremum at $w_i$ and $A^{\top}\Sigma_X^{-\frac12} w_i= e_i+O(\eta)$ (when being pedantic, we apply the result to permuted data changing coordinates $i$ and $d$). 
Then the matrix $W$ with rows $w_i=W^\top e_i$ is the desired matrix.
\end{proof}
The proof of Corollary~\ref{co:convergence} follows from standard results in convex optimization.
\begin{proof}[Proof of Corollary~\ref{co:convergence}]
It is well known that gradient descent locally converges for a convex function and sufficiently small step size, a proof for the general case where we optimize over a Riemannian manifold (in our case the sphere) can be found, e.g., in 
\cite{boumal_2023}. Local convergence of projected gradient descent can also be shown.
\end{proof}
Finally, we  prove that $WX$ essentially recovers the true sources
as stated in Theorem~\ref{th:mcc1}.
\begin{proof}[Proof of Theorem~\ref{th:mcc1}]
We verify the assumption of Lemma~\ref{le:mcc2} for $S$ and 
\begin{align}
\hat{S}=
W\Sigma_X^{-\frac12}X=W\Sigma_X^{-\frac12}(AS+\eta h(S)),
\end{align}
i.e., we define $T(s)=(W\Sigma_X^{-\frac12}A)s+\eta W\Sigma_X^{-\frac12} h(s))$.
Since we assume that $S_i$ have unit variance, the bound \eqref{eq:mcc2_ass1} holds with $c_1=1$.
Let us set $P=W\Sigma_X^{-\frac12}A$ and verify condition \eqref{eq:mcc2_ass2}
for $P$.
By construction of $W$ we have 
\begin{align}
e_i^\top P=e_i^\top W\Sigma_X^{-\frac12}A
=w_i^\top \Sigma_X^{-\frac12}A= e_i+O(\eta).
\end{align}
This implies that when setting $\rho(i)=i$ then we find
\begin{align}
\frac{1}{|P_{i,i}|^2} \sum_{j\neq i} |P_{i,j}|^2\leq \frac{O(\eta^2)}{1-O(\eta)}=O(\eta^2)
\end{align}
for sufficiently small $\eta$. Finally, we use the bound 
\begin{align}
\lVert \eta W\Sigma_X^{-\frac12} h(S)\rVert_{\P,2}
\leq \eta |W|\cdot |\Sigma_X^{\frac12}|\cdot \lVert h\rVert_{\P, q}=O(\eta).
\end{align}

Here we used that $|W|=d$
since its rows are normalized.
This implies 
\begin{align}
\frac{\lVert (\eta W\Sigma_X^{-\frac12} h)_i(S)\rVert_{\P,2}}{
|P_{i,i}| \lVert Z\rVert_{\P,2}}\leq O(\eta)
\end{align}
and therefore \eqref{eq:mcc2_ass3} holds with $c_3=O(\eta)$.
Now we can apply Lemma~\ref{le:mcc2} and conclude that
\begin{align}
\mcc(S,\hat{S})\geq 1-C\eta^2
\end{align}
for some constant $C>0$.
\end{proof}

\section{Proof for approximate identifiability 
for ICA with almost locally isometric mixing}\label{app:approx_iso_ica}
In this section we prove Theorem~\ref{th:local_ica_compact} and
then provide some extensions that allow us to consider $\P$ with unbounded support. 
However,  we first prove one simple auxiliary lemma which allows us 
to transform a data representation $X=AS+h(S)-b$ 
to $X-b' = A'S+h'(S)$ such that $h'$ is centered and $\E(Sh'(S))=0$.

\begin{lemma}\label{le:center}
    Suppose $S$ has distribution $\P$ such that  $\E(S)=0$
    and $\E SS^\top \geq \alpha^{-1} \mathrm{Id}$ in the sense of symmetric matrices for some $\alpha>0$.
    We assume that
    $X=AS+h(S)+b$ for an orthogonal matrix $A\in \SO(d)$, some function $h:\R^d\to\R^d$ and $b\in \R^d$.
    Then there is a linear map $A'$, $b'\in \R^d$ and $h':\R^d\to \R^d$
    such that $X-b' = A'S+h'(S)$ and
    \begin{align}
        \E_S(h'(S))=0,\quad
        \E_S(S^\top h'(S))=0.
    \end{align}
    Morevoew, for $q\geq 2$,
    \begin{align}
        \lVert h'\rVert_{\P,q}<
        2\left(1 +\alpha\lVert S\rVert_{\P,q}^2\right) \lVert h(S)\rVert_{\P,q}.
    \end{align}
    and 
    \begin{align}
        \sigma_{\min}(A')\geq 1-2\alpha \lVert S\rVert_{\P,2}\cdot \lVert h(S)\rVert_{\P,q}.
    \end{align}
\end{lemma}
\begin{proof}
First, we find 
\begin{align}
    X - \E X = AS + (h(S)+b - \E X)
    = AS + h_c(S)
\end{align}
where we defined $h_c(S)=h(S)+b - \E X$.
Note that $\E h_c(S)=\E X-\E X- A\E S =0$ and thus
$h_c(S)=h(S)-\E h(S)$.
By the triangle inequality and H\"older's inequality (which implies
$\lVert f\rVert_{\P,p}\leq \lVert f\rVert_{\P,q}$ for $p<q$) we get
\begin{align}
\lVert h_c \rVert_{\P,q}=\lVert h - \E_\P h\rVert_{\P,q}\leq
\lVert h\rVert_{\P,q}+\lVert \E_\P h\rVert_{\P,q}
\leq \lVert h\rVert_{\P,q}+|\E_\P h|\leq 2\lVert h\rVert_{\P,q}.
\end{align}
Thus we find that $h_c$ is centered, and its $q$ norm is at most larger by a factor of 2 than the corresponding norm of $h$. Then we define
\begin{align}
h'(S)=h_c(S) - \E(h_c(S)S^\top)\E(SS^\top)^{-1} S,\quad
A'=A + \E(h_c(S)S^\top)\E(SS^\top)^{-1}
\end{align}
i.e., we regress $h_c(S)$ linearly on $S$  so that $\E(h'(S)S^\top) = 0$. 
We bound using Cauchy-Schwarz and $q\geq 2$
\begin{align}
\begin{split}
\lVert \E(h_c(S)S^\top)\E(SS^\top)^{-1}S\rVert_{\P,q}
&\leq \alpha |\E(h_c(S)S^\top)|\lVert S\rVert_{\P,q}
\leq \alpha \E(|h_c(S)S^\top|)\lVert S\rVert_{\P,q}
\\
&\leq \alpha\E(|h_c(S)|\cdot |S|)\lVert S\rVert_{\P,q}
\leq \alpha\lVert h_c(S)\rVert_{\P,2}\lVert S\rVert_{\P,2}\lVert S\rVert_{\P,q}
\leq \alpha\lVert h_c(S)\rVert_{\P,q}^2\lVert S\rVert_{\P,q}.
\end{split}
\end{align}
This implies
\begin{align}
\lVert h'(S)\rVert_{\P,q}\leq \left(1 +\alpha\lVert S\rVert_{\P,q}^2\right) \lVert h_c(S)\rVert_{\P,q}
\leq  2\left(1 +\alpha\lVert S\rVert_{\P,q}^2\right) \lVert h(S)\rVert_{\P,q}.
\end{align}
Finally we conclude the lower bound on $A'$. 
We observe that 
\begin{align}
|A'-A|=|\E(h_c(S)S^\top)\E(SS^\top)^{-1}|\leq \alpha \lVert 
h_c(S)\rVert_{\P,2}\cdot \lVert S\rVert_{\P,2}
\leq 2\alpha \lVert S\rVert_{\P,2}\cdot \lVert h(S)_{\P,q}.
\end{align}
Using  $A\in \SO(d)$ we conclude that 
\begin{align}
    \sigma_{\min}(A')\geq 1-2\alpha \lVert S\rVert_{\P,2}\cdot \lVert h(S)_{\P,q}.
\end{align}
\end{proof}
We can now prove Theorem~\ref{th:local_ica_compact}.
\begin{proof}[Proof of Theorem~\ref{th:local_ica_compact}]
    The proof essentially follows by combining Theorem~\ref{th:almost_orthogonal} and Theorem~\ref{th:mcc1}.
    We note that by assumption the support of $\P$ is connected and by independence of the $S_i$ it is actually a cuboid and thus a Lipschitz domain.
    Using Theorem~\ref{th:almost_orthogonal} we find 
    that if we define  $g$ as in 
    \eqref{eq:def_of_g} (or \eqref{eq:def_of_g2} for $d<D$), then 
    $\tilde{X}=g^{-1}\circ f(S)=g^{-1}(X)$ can be expressed
    as
    \begin{align}
        \tilde{X} = g^{-1}\circ f(S)=
        AS+h(S)+b
    \end{align}
   for some $A\in \SO(d)$  and $h$ satisfies
   the bound 
\begin{align}
\lVert h\rVert_{\P,q} \leq C_1 \sodist_p(f, \Omega)
\end{align} 
where $p$ and $q$ are as in the statement of the Theorem, i.e, 
$q=\max(3,d_g)$ and $p=dq/(d+q)$ which is equivalent to $q=pd/(d-p)$,
as required in Theorem~\ref{th:almost_orthogonal}.
    By shifting $\tilde{X}$ we can assume that $\tilde{X}$ is centered.
    Now we apply Lemma~\ref{le:center} from above and find that 
    \begin{align}
        \tilde{X}=A'S+h'(S)
    \end{align}
    where 
    \begin{align}
         \sigma_{\min}(A')\geq 1-2\alpha \lVert S\rVert_{\P,2}\cdot \lVert h(S)\rVert_{\P,q}\geq 1 - 2\alpha \lVert S\rVert_{\P,2}C_1\sodist_p(f, \Omega)>\frac12
    \end{align}
    if $\sodist_p(f,\Omega)$ sufficiently small.
    Moreover, there is a constant $C_1'>0$
    depending on $C_1$ and the distribution $\P$ such that
    \begin{align}
        \lVert h'\rVert_{\P, q}
       < C_1'\sodist_p(f, \Omega).
    \end{align}
    We can write 
\begin{align}
h(S)= (C_1'\sodist_p(f,\Omega))\cdot (C_1'\sodist_p(f,\Omega))^{-1}h(S)
=\eta \cdot \frac{h(S)}{C_1'\sodist_p(f,\Omega)}
\end{align}
where $\eta = C_1'\sodist_p(f,\Omega)$.
Now we apply Theorem~\ref{th:pert_matrix}.
If $\eta = C_1'\sodist_p(f,\Omega)<\eta_0$
 we find that there is a matrix $W\in \R^{d\times d}$ with rows $w_i$ such that $H$ (defined using the distribution of $\tilde{X}$) has local extrema at $w_i$.
 Moreover, by Theorem~\ref{th:mcc1}
the reconstructed sources $\hat{S}=W\tilde{X}=W\tilde{g}(X)$
satisfy 
\begin{align}
\mcc(\hat{S},S)\geq 1 - C_2\eta^2
\geq 1 - C_3\sodist_p^2(f,\Omega)
\end{align}
for $C_3=C_2 C_1'^2$.
\end{proof}

We can also extend the results to ICA with approximately locally isometric 
mixing function and unbounded support of the sources. 
 To simplify the analysis, we restrict our attention to functions that are perfect local isometries away from a compact set, i.e., we consider
%TODO adapt when moving back
\begin{align}
\begin{split}
\mc{F}_{c-\iso}(d,& D,\Omega)=\{\text{$f:\R^d\to M\subset \R^D:$ $f|_{\Omega^\complement} $ is a }
\\
&\text{local isometry and  $f(\omega)$ is not 
 }
\\
&
\text{isometric to $(\R^d,g_{\mathrm{stand}})$  for any $\omega\subset \Omega$
}\}.
\end{split}
\end{align}
Here $g_{\mathrm{stand}}$ denotes the standard metric on $\R^d$.
The definition essentially ensures that $\Omega$ is the maximal set such that outside of $\Omega$ the function $f$ is a local isometry. Note that 
the second part of the condition is satisfied, e.g., when the curvature of $f(\Omega)$ is everywhere non-vanishing.
Then we have the following result.
\begin{theorem}\label{th:iso_ica}
Suppose that the mixing  $f\in \mc{F}_{c-\iso}(d, D,\Omega)$ for some
bounded set $\Omega$ such that $\Omega^\complement$ is connected.
Assume that $X=f(S)$ where $S\sim \P$. 
Assume that $\P$ of $S$ satisfies Assumption~\ref{as:ICA1}, \ref{as:ICA3}, and \ref{as:ICA4} for some contrast function $G$ and the density of $\P$ is bounded above and everywhere positive on $\R^d$ and lower bounded  on $\Omega$.
Let $q=\max(d_g,3)$ and $p=dq/(d+q)$.
Assume that $\sodist_p(f,\Omega)$ is sufficiently small.
Then we can find a transformation   $\tilde{g}:M\to \R^d$
such that the transformed data $\tilde{X}=\tilde{g}(X)$ is centered
and
has the property that there is a matrix $W\in \R^{d\times d}$ whose rows are $d$ normalized vectors $w_i$ which are 
local extrema of the function $w\to \E(G(w^\top \Sigma_{\tilde{X}}^{-\frac12} \tilde{X}))$  such that $\hat{S}=W\tilde{X}$
satisfies
\begin{align}\label{eq:mcc_final_thm2}
\mcc(\hat{S},S)\geq 1-C\sodist_p^2(f,\Omega)
\end{align}
where $C$ depends on $\Omega$, $d$, and the parameters in the Assumptions~\ref{as:ICA1}, \ref{as:ICA3}, and \ref{as:ICA4}.
\end{theorem}
\begin{remark}
\begin{itemize}
\item    Of course, we expect to recover the vectors  $w_i$ typically by applying an algorithm for linear ICA to the transformed data $\tilde{X}$, however, due to the lack of theoretical guarantees for 
    ICA algorithms we cannot show a stronger statement (see Remark~\ref{rmk:unsatisfactory}).
    \item
 We expect that the   assumption that $f$ is a local isometry outside $\Omega$ 
 can be relaxed, e.g., by  using bounded contrast functions like $G(s)=\exp(-s^2/2)$.
 \end{itemize}
\end{remark}
%\section{Proof of Theorem~\ref{th:iso_ica}}
We now prove Theorem~\ref{th:iso_ica}. 
In its proof we need a slightly different rigidity result in this case (actually a simpler result because we in addition fix the boundary) which we now state.
We denote by $\sigma_{\max}$ the largest singular value of a matrix.
\begin{theorem}[Theorem~2.1 in \citet{kohn1982new}]\label{th:rigidity2}
Let $ 1\leq p\leq d$, suppose there is a rigid motion $L(z)=Az+b$ with $A\in \SO(d)$ and $u\in W^{1,p}_{\mathrm{loc}}(\R^d,\R^d)$. We assume that there
is a bounded set $\Omega$ such that $u(z)=L(z)$ for $z\notin  \Omega$. Then
\begin{align}
\lVert u-L\rVert_{q}\leq C(d,p)\lVert (\sigma_{\max}(Du)-1)_+\rVert_p
\end{align} 
where $q=np/(n-p)$.
\end{theorem}
\begin{remark}
\begin{enumerate}
\item
Note that by Lemma~\ref{le:sod} we can bound
\begin{align}
(\sigma_{\max}(Du)-1)_+\leq \sqrt{\sum_i (\sigma_i-1)^2}=\dist(A,\SO(d)).
\end{align}
\item The exponent $q$ agrees with the exponent for the Sobolev embedding result.
\item There is an extension to $p>d$ where $\lVert u-L\rVert_\infty$ can be bounded.
\item The result Theorem~\ref{th:rigidity} is more general than this result because it does not assume affine boundary values. 
We cannot directly apply Theorem~\ref{th:rigidity} in our context because we need
that the linear map in  \eqref{eq:rigidity_bound} 
can be chosen as  the affine boundary values $L$. While this is the case, it requires to inspect the proof carefully, and it is simpler to rely on this
known result from the literature.
\end{enumerate}
\end{remark}
We can now continue with the Proof of Theorem~\ref{th:iso_ica}. To clarify the proof structure, we split the proof into two parts. First, we  establish that the assumptions of the Theorem imply that we can reduce the problem to perturbed linear ICA. 
We summarize this result in the following proposition.
\begin{proposition}\label{prop:rigid}
Suppose that the mixing  $f\in \mc{F}_{c-\iso}(d, D,\Omega)$ for some
bounded set $\Omega$ and $\Omega^\complement$ connected.
Suppose that $X=f(S)$ with $S\sim\P$.
Assume that the density of $S$ is globally positive and upper bounded and, moreover, lower bounded 
on $\Omega$. Then we can find a map $g:M\to \R^d$
such that the recovered sources $X'=g(X)$ satisfies
\begin{align}\label{eq:hat_S}
X'=AS+h(S)
\end{align}
for some rotation $A\in \SO(d)$ and function $h$ such that
\begin{align}
\lVert h\rVert_q\leq C\sodist_p(f,\Omega)
\end{align}
where $q=pd/(d-p)$ and $C$ depends on $d$, $p$, and the bounds on the density of $\P$ on $\Omega$.
% If in addition $\lVert S\rVert_{\P,q}<\infty$  and $q\geq 2$
% we can decompose
% \begin{align}
% X'=A'S + h'(S)
% \end{align}
% such that $h'$ is centered,
% i.e., $\E_\P h(S)=0$ and, $\E(h_i'(S)S_j)=0$ for $1\leq i,j\leq d$ and 
% \begin{align}
% \lVert h'\rVert_q\leq C\sodist_p(f,\Omega)
% \end{align}
% where $C$ now in addition depends on $\lVert S\rVert_{\P,q}$ and
% $\sigma_{\min}(A')\geq 1-C\sodist_p(f,\Omega)$ for a (different) constant $C$.
\end{proposition}
\begin{proof}
For a diffeomorphism $g:\R^d\to M$ we let $I(g)\subset M$ be the maximal open set such that $g^{-1}|_{I(g)}$ is a local isometry. We consider the set $\mc{G}$ of diffeomorphisms from $\R^d$ to $M$ such that $I(g)$ is maximal. 
By definition of the function class $\mc{F}_{c-\iso}(d, D,\Omega)$
there are functions $g$ such that $I(g)=f(\mathring{\Omega}^\complement)$ 
(e.g., $g=f$), moreover $I(g)\cap f(\Omega)=\emptyset$ 
for any $g$ by definition
of $\mc{F}_{c-\iso}(d, D,\Omega)$. Therefore, $\mc{G}$ is well-defined.
Now we consider $g\in \mc{G}$ such that
\begin{align}
g \in \argmin_{\bar{g}\in \mc{G}} \int_{\R^d} \dist^p (D\bar{g}^{-1}(x), \SO(T_xM,d))\, f_\ast\P(\d x).
\end{align}
By shifting $g$ we can assume that $g^{-1}\circ f(S)$ is centered.
Then we consider the transition function $T=g^{-1}\circ f$.
By assumption $f|_{\Omega^\complement}$ is a local isometry, and
we have shown that $g^{-1}|_{f(\Omega^\complement)}$ is a local isometry and thus 
$T|_{\Omega^\complement}$ is a local isometry and therefore $T(s)=As+b=L(s)$ for 
some $A\in \SO(d)$ by Theorem~\ref{th:isometry}.
We now consider the restriction $T|_\Omega$
and apply Lemma~\ref{le:simple_bound_T2} to find
\begin{align}
\begin{split}
\int_\Omega \dist^p&(DT(s),\SO(d))\,\d s
\\
&\leq 
3^{p}\int_\Omega  \mathrm{dist}^p(Df(s) ,\mathrm{SO}(d, T_{f(s)}M))\, \d s
+C\int_{\Omega'}  \mathrm{dist}^p(D g^{-1}(g(s)) ,\mathrm{SO}(T_{f(s)}M, d))
\, \Q(\d s).
\end{split}
\end{align}
Now since $f\in \mc{G}$ we can conclude (using the upper bound on the density of $\P$) as before that
\begin{align}
\begin{split}
\int_{\Omega'}  \mathrm{dist}^p(D g^{-1}(g(s)) ,\mathrm{SO}(T_{f(s)}M, d))
\, \Q(\d s)
&\leq 
\int_{\Omega}  \mathrm{dist}^p(D f^{-1}(f(s)) ,\mathrm{SO}(T_{f(s)}M, d))
\, \P(\d s)
\\
&\leq C
\int_{\Omega}  \mathrm{dist}^p(D f^{-1}(f(s)) ,\mathrm{SO}(T_{f(s)}M, d))
\, \d s.
\end{split}
\end{align}
Combining the last two displays we find
\begin{align}
\int_\Omega \dist^p(DT(s),\SO(d))\,\d s\leq C \sodist_p^p(f,\Omega).
\end{align}
Now we apply Theorem~\ref{th:rigidity2} (and the remark below this result) to conclude that 
\begin{align}
\lVert T-L\rVert_q \leq C \sodist_p(f,\Omega).
\end{align}
Defining $h=T-L$ and using the upper bound on the density of $\P$ we conclude that $X'=g^{-1}(X)=T(S)$ is given by
\begin{align}
X=L(S)+h(S)
\end{align}
and 
\begin{align}
\lVert h\rVert_{\P,q}\leq C \sodist_p(f,\Omega).
\end{align}
\end{proof}
Now we can complete the proof of Theorem~\ref{th:iso_ica}.

\begin{proof}[Proof of Theorem~\ref{th:iso_ica}]
The proof is now the same as the proof of Theorem~\ref{th:local_ica_compact} except that we first apply Proposition~\ref{prop:rigid} instead of Theorem~\ref{th:almost_orthogonal}.
Indeed, by applying Proposition~\ref{prop:rigid}
and the definition of $g$ used in the proof we find that
\begin{align}
\tilde{X}={g}(X)=g(f(S))=AS+h(S)+b
\end{align}
where $A\in \SO(d)$ and $\lVert h\rVert_{\P,q}\leq C\sodist_p(f,\Omega)
= C\sodist_p(f, \R^d)$.
The rest of the proof is the same as the proof of Theorem~\ref{th:local_ica_compact}, i.e., we apply Lemma~\ref{le:center}
to rewrite $\tilde{X}=A'S+h'(S)$ where $h'(S)$ is centered and
$\E h'(S)S^\top)=0$ and then apply Theorem~\ref{th:mcc1}. Note that
those two results do not assume bounded support of $S$.
\end{proof}

\section{Experimental Illustration of non-robustness of piecewise linear functions}\label{app:illustration}

% \begin{figure*}[ht!]
% \begin{center}
% \includegraphics[width=\linewidth]{concise.png}
% %/Users/sbuchholz/PycharmProjects/perturbed_ica/plots/distances.pdf}
% \end{center}
% %\vspace{-.3cm}
% \caption{
% (Left) Color map of Gaussian latent variable $Z$, (Center) Color map of the transformed data $X=f(Z)$ where $f$ is a piecewise linear approximation of a radius dependent rotation (i.e., $f(Z)\overset{\mathcal{D}}{\approx} X$), (Right) Representation $X'=f'(Z)$ learned by a VAE with ReLU activation functions initialized with $f$ and $f^{-1}$ for decoder and encoder 
% and small variance.}
% \label{fig:2}
% \end{figure*}

In this appendix we collect some additional details on the experimental illustration of the non-robustness of learning piecewise linear mixing functions. Recall that we consider $Z$ following a two-dimensional standard normal distribution. We let 
$m$ be a radius dependent rotation \cite{hyvarinen1999nonlinear}
which is given by 
\begin{align}
    m(z_1, z_2)
    =
    \begin{pmatrix}
    \cos(5  \p(z^2))z_1 + \sin(5\p(z^2))z_2
    \\
    -\sin(5\p(z^2))z_1 + \cos(5\p(z^2))z_2
    \end{pmatrix}
\end{align}
where $\p:\R\to\R$ is a smooth bump function supported in $(0,1)$ and 
given by
\begin{align}
    \p(t) =
    \begin{cases}
        \exp\left(-\frac{1}{1-(2t-1)^2}\right) &\text{for $0<t<1$}
        \\
        0 & \text{else}.
    \end{cases}
\end{align}
It can be checked explicitly that $m(Z)\indistribution Z$.
Now we consider a VAE model with encoder and decoder each having 3 hidden residual layers with 64 neurons. We train encoder and decoder in a supervised way on pairs $(m(Z), Z)$ and $(Z,m(Z))$ (where $Z$ follows a standard normal distribution) using the mean squared error until the mean error is smaller than $0.001$. 
We define the map implemented by the decoder as $f$. 
We generate data $f(Z)$ where $Z$ is Gaussian and then 
train a VAE on samples $f(Z)$ and we initialize the VAE with the learned $f$ and $f^{-1}$ and give a large negative bias  of $-5$ to the log-variance layer to make the encoder close to deterministic.

\section{Experimental Illustration of Convergence Rate for perturbed linear ICA}\label{sec:experiment}
\begin{figure*}[!ht]
\begin{center}
\includegraphics[width=\linewidth]{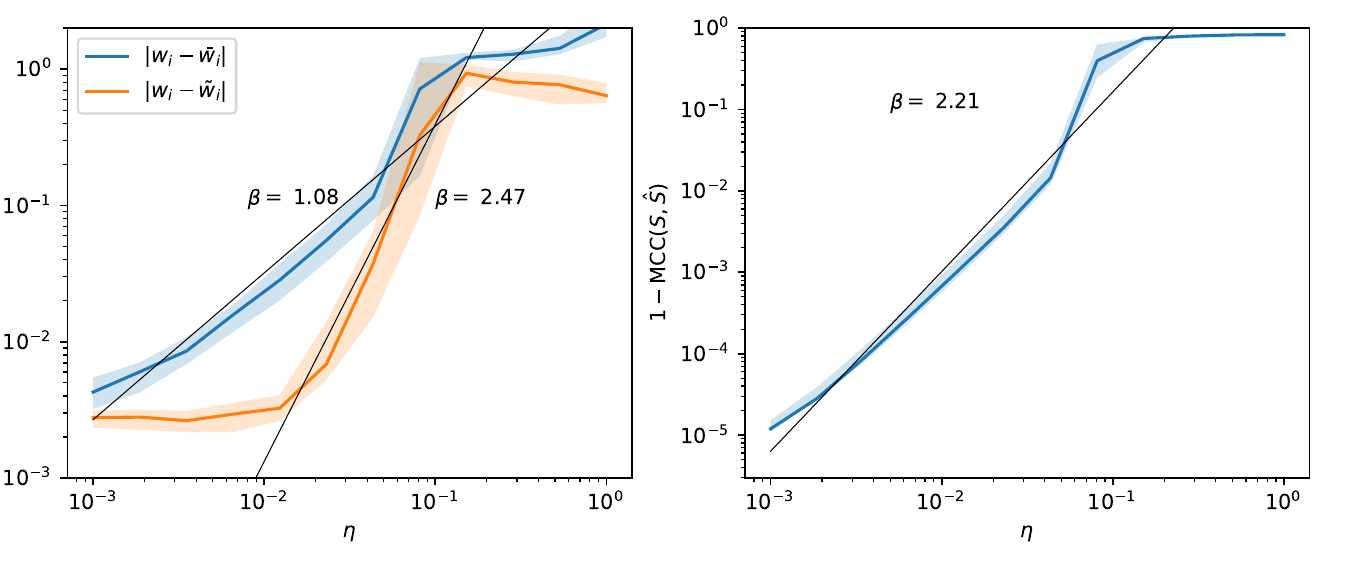}
%/Users/sbuchholz/PycharmProjects/perturbed_ica/plots/distances.pdf}
\end{center}
%\vspace{-.3cm}
\caption{
(Left) Distance of recovered linear unmixing $w$ from $\bar{w}$ (see \eqref{eq:def_barw}) and $\tilde{w}$ (see \eqref{eq:tilde_wd})
as a function of $\eta$. Plotted is the median over $5$ runs with $d=5$
(i.e., 25 recovered components).
The shaded area shows the range from the $32\%$ to the $68\%$ quantile.
(Right) Difference of the Mean $\mcc$ over 5 runs of $\hat{S}=W\Sigma^{-\frac12}_XX$ from  perfect recovery ($\mcc=1$).
Regression lines are obtained by linear regression in log-log space for $0.01\leq \eta\leq 0.1$ and $\beta$ values indicate their slope.}
\label{fig:1}
\end{figure*}
To illustrate the results in Section~\ref{sec:perturbed}, we provide a small scale experimental illustration. We consider data generated by $X=AS+\eta h(S)$ 
where $S_i$ follow independent Laplace distributions
and all entries of $A$ follow a centered normal distribution with variance $d^{-1}$ and  
$h_i(S)\propto S_i^3-\beta_i S_i - \alpha_i$ where $\alpha_i$, $\beta_i$ are chosen such that $\E(h_i(S)S_j)=0$, $E(h_i(S))=0$ and the proportionality 
is chosen such that $\E(h_i(S)^2)=1$.
Then we run the Fast-ICA algorithm \citep{hyvarinen1999fast} on $\hat{\Sigma}_X^{-\frac12}X$
with $n=10^6$ samples for $\eta\in [10^{-3},1]$
We initialize $w_{\mathrm{init}}=\bar{w}_i$ (recall that $\bar{w}_i$ 
denote the rows of $A^{-1}(AA^\top)^\frac12$, i.e., the true unmixing of the whitened linear part $A$) to avoid problems with spurious minima 
and matching of minima which are already present for linear ICA. While $\bar{w}_i$ 
cannot be inferred from $X$, this allows us to focus here on the essential differences of our setting to linear ICA.
Recall that $\tilde{w}_i$ was defined in \eqref{eq:tilde_wd} and denotes
the predicted next order expression of the recovered source.
In Figure~\ref{fig:1} we plot the distance of
the output $w_i$ of the fast-ICA algorithm to $\bar{w}_i$  and $\tilde{w}_i$ and evaluate the MCC of the recovered sources $\hat{S}$ and the true sources $S$ as a function of $\eta$, the strength of the perturbation.

We find that the dependence on $\eta$ roughly matches our theoretical results.
Note that for $\eta\gtrsim .1$ the error is already of order 1  and
thus independent of $\eta$ while for $\eta\lesssim .01$ the distance
$|w_i-\tilde{w}_i|$ no longer decreases due to final sample effects
(note that roughly $|w_i-\tilde{w}_i|=O(n^{-\frac12})$ for such $\eta$).
We also refer to \cite{horan2021when} where small scale real-world image data was 
investigated and the independent components recovered. Their mixing function is only approximately locally isometric, and therefore our results give a theoretical justification for their observations.

\section{A Class of approximately isometric random Functions}\label{app:random_functions}
In this section we illustrate how approximate local isometries naturally appear when considering random high dimensional functions. Note that a slightly more general class of random functions was considered in \citet{reizinger2023independent}. However, their work does not  quantify approximate local isometries.
We 
consider random functions $f:\R^d\to \R^D$ where we are mostly interested in the case $D\gg d$.
Assume that each coordinate $f_i$ is given by a random draw 
from a smooth Gaussian process, e.g., a mean zero process whose covariance is given by a
Mat\'ern kernel. Assume that the kernel is isotropic $K(x,y)=K(|x-y|)$
with $K'(0)=0$ and $K''(0)=-1$ (can be obtained by scaling).
Then the following holds for $k\neq l$
\begin{align}
\begin{split}
\mathbb{E}(\partial_k f_i(x)\partial_lf_i(x))
&= 
\lim_{h\to 0} h^{-2}\mathbb{E}( (f_i(s+he_k)-f_i(s))(f_i(s+he_l)-f_i(s)))
\\
&=
\lim_{h\to 0} h^{-2}\mathbb{E}( (f_i(s+he_k)-f_i(s))(f_i(s+he_l)-f_i(s)))
\\
&=
\lim_{h\to 0} h^{-2}(K(\sqrt{2}h)-2K(h)+K(0)).
\end{split}
\end{align}
Here, we used in the last step that $|(s+he_k)-(s+he_l)|=h|e_k-e_l|=\sqrt{2}h$.
We then find using Taylor expansion of $K$ (recall $K'(0)=0$) that
\begin{align}
\begin{split}
    \mathbb{E}(\partial_k f_i(s)\partial_lf_i(s))
    =\lim_{h\to 0} h^{-2}\left((K(0)+h^2 K''(0)+o(h^2)) -2\left(K(0)-\frac{h^2}{2} K''(0)\right) + o(h^2)) +K(0)\right)
    =0.
    \end{split}
\end{align}
We also note that 
\begin{align}
\begin{split}
\mathbb{E}(\partial_k f_i(s)\partial_kf_i(s))
&= 
\lim_{h\to 0} h^{-2}\mathbb{E}( (f_i(s+he_k)-f_i(s))(f_i(s+he_k)-f_i(s)))
\\
&=
\lim_{h\to 0} h^{-2}\mathbb{E}( 2K(0)-2K(h))=
-K''(0) =1.
\end{split}
\end{align}
We conclude that the matrix $Df(s)$ has independent Gaussian entries with
variance $1$. 
To measure how isometric $f$ is, we need to investigate the singular values of
$Df$. This has been investigated extensively in the field of random matrix 
theory, e.g., the random matrices appearing here are a special case of the Wishart distribution. Indeed, there  are explicit expressions for the joined eigenvalue density of $D^{-1}X^\top X$ where $X_{ij}$
follow independent standard normal distributions and fine information on the spectrum is available even when $d,D\to\infty$ jointly with $d/D$ fixed. Here we are only interested in the asymptotic scaling as $D\to \infty$, which is much simpler to obtain, and we now sketch this to provide some intuition. We denote $A=\frac{1}{D}Df(s)^\top Df(s)$, and we introduce $\tilde{f}=\sqrt{D}^{-1}f$.
Using the results above, we find
\begin{align}
  \E A_{kl}=\E  (D\tilde{f}^\top(s) D\tilde{f}(s))_{kl}=\frac{1}{D}\sum_m
\E    (Df(s))_{mk} (Df(s))_{ml}
    =\frac{1}{D}
\sum_m  \E  \partial_k f_m(s)\partial_l f_m(s)
= \delta_{kl}.
\end{align}
In other words $\E A=\id_d$.
By the law of large numbers we infer that
\begin{align}
\lim_{D\to \infty} D\tilde{f}^\top (s) D\tilde{f}(s)
=\id_d.
\end{align}
In other words, we conclude that as $D\to \infty$ the sequence $\tilde{f}$
becomes increasingly isometric. 
This can also be quantified.
Note that for $k\neq l$
\begin{align}
    \E (\partial_k f_i(s)\partial_l f_i(s))^2=
    \E (\partial_k f_i(s))^2
    \E (\partial_l f_i(s))^2=1.
%     \begin{cases}
% 1 \quad\text{if $k\neq l$}
% \\
% 3 \quad \text{if $k=l$}
%     \end{cases}  
\end{align} 
Here we used that uncorrelated Gaussian variables are independent.
Similarly, we obtain for $k=l$
\begin{align}
      \E (\partial_k f_i(s)\partial_k f_i(s))^2
      - \left(\E (\partial_k f_i(s)\partial_k f_i(s))\right)^2
      =
    \E (\partial_k f_i(s))^4-1=3-1=2.
\end{align}
Linearity of the variance implies that
\begin{align}
    \mathrm{Var} A_{kl} \leq  \frac{D}{D^2}(1+\delta_{kl})=\frac{1}{D}
    (1+\delta_{kl}).
\end{align}
We conclude that 
\begin{align}
    \E |A - \E(A)|_F^2 = \frac{d^2+d}{D}
\end{align}

% Using tail bounds for sub-Gaussian variables one finds then that
% \begin{align}
%     \P (|A_{kl}-\E(A_{kl})|>t)\leq e^{-cD t^2}.
% \end{align}
% Applying a union bound we conclude that
% \begin{align}
%      \P (|A-\E(A)|_F>t)\leq Ce^{-cD t^2}.
% \end{align}
Let $\sigma_i\geq 0$ be the singular values of $D\tilde{f}(s)$.
Then we bound $|\sigma_i-1|\leq |\sigma^2_i-1|$ for $\sigma_i>0$ and find
\begin{align}
   \sum_i |\sigma_i-1|^2\leq \sum_i |\sigma_i^2-1|^2
    = |A-\E(A)|^2_F
\end{align}
Using \eqref{eq:dist_so_under} we conclude that
\begin{align}
   \E \dist^2(D\tilde{f}(s),\SO(d, T_{f(s)}M))
   = \E \sum_i (\sigma_i-1)^2
   \leq \E |A-\E(A)|^2=\frac{d^2+d}{D}.
\end{align}
The argument for the second term in \eqref{eq:dist_f_iso_general2}
is similar but a bit more involved because we need to bound the inverse. 
Suitable bounds for $A^{-1}$ can be found, e.g., in Theorem~2.4.14 in \cite{kollo2006advanced}. They then imply that 
\begin{align}
      \E \dist^2(D\tilde{f}^{-1}(x),\SO(T_{x}M, d))
   \leq \frac{C}{D}.
\end{align}
We then conclude that
\begin{align}
    \E \sodist_2(\tilde{f}, \Omega) \leq  \sqrt{\E \sodist_2^2(\tilde{f}, \Omega)}\leq  \frac{C}{\sqrt{D}}
\end{align}
and therefore the functions $\tilde{f}$ approach local isometries as $D\to \infty$ in a quantitative way with the expected rate.
This result can be generalized to $p>2$ using stronger concentration bounds for sub-Gaussian variables.

\end{document}